\documentclass[letterpaper, 10 pt, journal, twoside]{ieeetran}  

\IEEEoverridecommandlockouts                              

\usepackage{amsmath}
\usepackage{amssymb}
\usepackage{amsfonts}
\usepackage{amsthm}
\usepackage{booktabs}
\usepackage{graphicx}
\usepackage{color}
\usepackage[absolute,overlay]{textpos}
\usepackage[font=footnotesize]{subfig}
\usepackage{fancyhdr}
\graphicspath{{./figures/}}
\usepackage{algorithm2e}
\newcommand{\sign}{\text{sign}}
\newtheorem{theorem}{Theorem}

\title{
Dynamic Median Consensus for Marine Multi-Robot Systems Using Acoustic Communication}

\author{Goran Vasiljevi{\'{c}}$^{1}$, Tamara Petrovi{\'{c}}$^{1}$, Barbara Arbanas$^{1}$ and Stjepan Bogdan$^{1}$
\thanks{Manuscript received: February, 24, 2020; Revised May, 27, 2020; Accepted June, 17, 2020.}
\thanks{This paper was recommended for publication by Editor Jonathan Roberts upon evaluation of the Associate Editor and Reviewers' comments. This work was in part supported by the EU-H2020-FETPROACT-2-2014 project subCULTron - Submarine Cultures Perform LongTerm Robotic Exploration of Unconventional Environmental Niches, grant agreement No. 640967 and EU-H2020 CSA project AeRoTwin - Twinning coordination action for spreading excellence in Aerial Robotics, grant agreement No. 810321. The work of B. Arbanas was supported in part by the “Young Researchers' Career Development Project–Training of Doctoral Students” of the Croatian Science Foundation funded by the European Union from the European Social Fund.}
\thanks{$^{1}$Goran Vasiljevi{\'{c}}, Tamara Petrovi{\'{c}}, Barbara Arbanas and Stjepan Bogdan are with University of Zagreb
Faculty of Electrical Engineering and Computing, Laboratory for Robotics and Intelligent Control Systems (LARICS), Unska 3, Zagreb 10000, Croatia; 
        {\tt\small goran.vasiljevic@fer.hr}}
\thanks{Digital Object Identifier (DOI): see top of this page.}
}

\begin{document}

\maketitle
\thispagestyle{empty}
\pagestyle{empty}
\maxdeadcycles=20000
\begin{textblock*}{14.9cm}(3.2cm,0.75cm) %
	{\footnotesize © 2020 IEEE.  Personal use of this material is permitted.  Permission from IEEE must be obtained for all other uses, in any current or future media, including reprinting/republishing this material for advertising or promotional purposes, creating new collective works, for resale or redistribution to servers or lists, or reuse of any copyrighted component of this work in other works.}
\end{textblock*}

\begin{abstract}
In this paper, we present a dynamic median consensus protocol for multi-agent systems using acoustic communication. The motivating target scenario is a multi-agent system consisting of underwater robots acting as intelligent sensors, applied to continuous monitoring of the state of a marine environment. The proposed protocol allows each agent to track the median value of individual measurements of all agents through local communication with neighbouring agents. Median is chosen as a measure robust to outliers, as opposed to average value, which is usually used. In contrast to the existing consensus protocols, the proposed protocol is dynamic, uses a switching communication topology and converges to median of measured signals. Stability and correctness of the protocol are theoretically proven. The protocol is tested in simulation, and accuracy and influence of protocol parameters on the system output are analyzed. The protocol is implemented and  validated by a set of experiments on an underwater group of robots comprising of aMussel units. This experimental setup is one of the first deployments of any type of consensus protocol for an underwater setting. Both simulation and experimental results confirm the correctness of the presented approach.

\end{abstract}

\section{Introduction}
\IEEEPARstart{O}{ne} of the envisioned goals of the subCULTron project was to develop a marine multi-robot system for intelligent long-term monitoring of underwater ecosystems \cite{subcultron}. The underwater system is comprised of 3 different types of robots (Figure \ref{fig:subCULTron_robots}). Artificial mussels (aMussels) are sensor hubs attached to the sea bottom, which monitor natural habitat, including biological agents like algae, bacterial incrustation, and fish. They serve as the collective long-term memory of the system, allowing information to persist beyond the runtime of other agents, enabling the system to continue developing from previously learned states \cite{Arbanas20182}, \cite{arbanas2018}. On the water surface, artificial lily pads (aPads) interface with humans, delivering energy and information influx from ship traffic or satellite data \cite{babic2018}. Between those two layers, artificial fish (aFish) move, monitor and explore the environment and exchange information with aMussels and aPads. 

\begin{figure}
    \centering
    \subfloat{
	    \includegraphics[height=0.25\linewidth]	{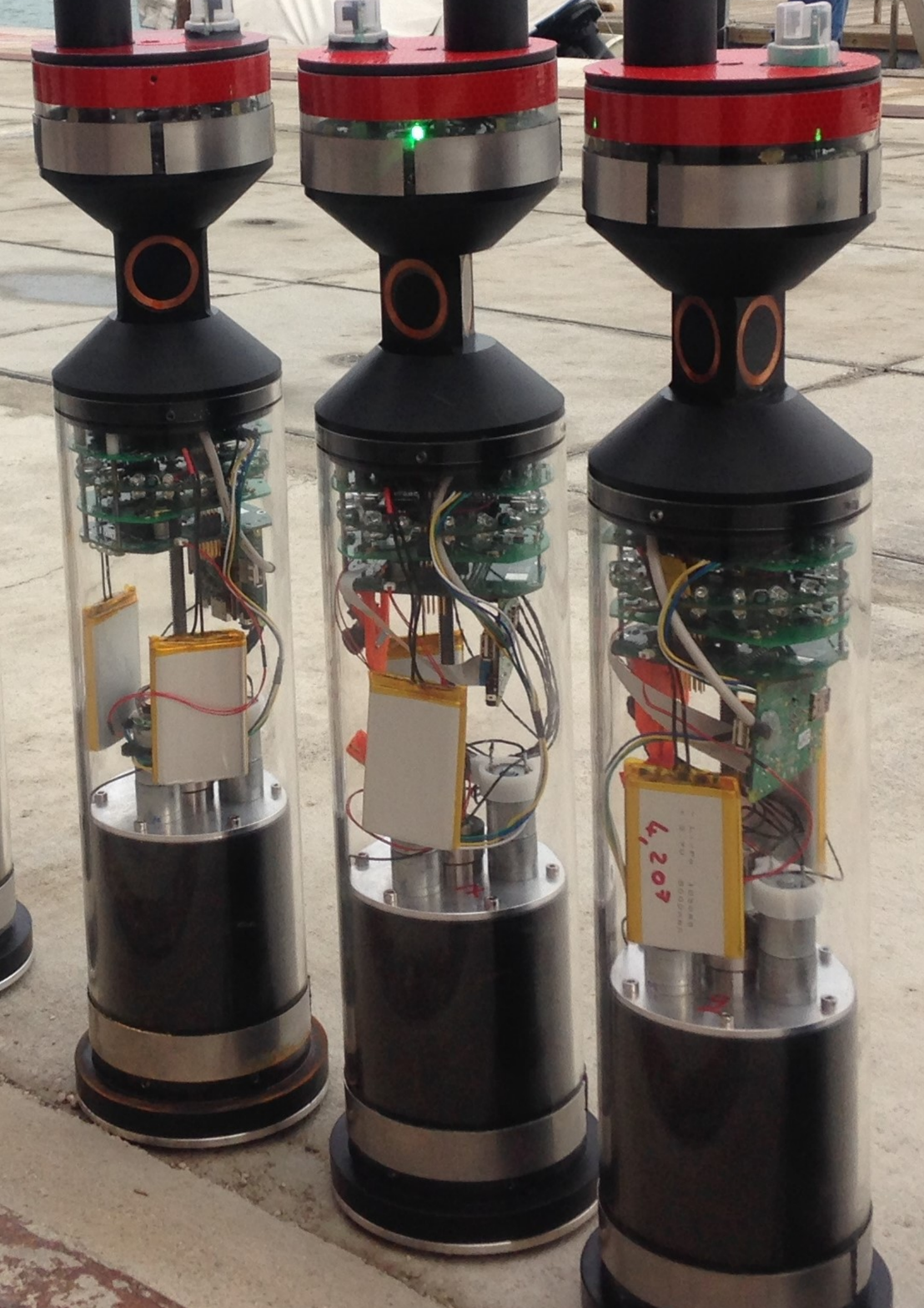}}
    \subfloat{
	    \includegraphics[height=0.25\linewidth]{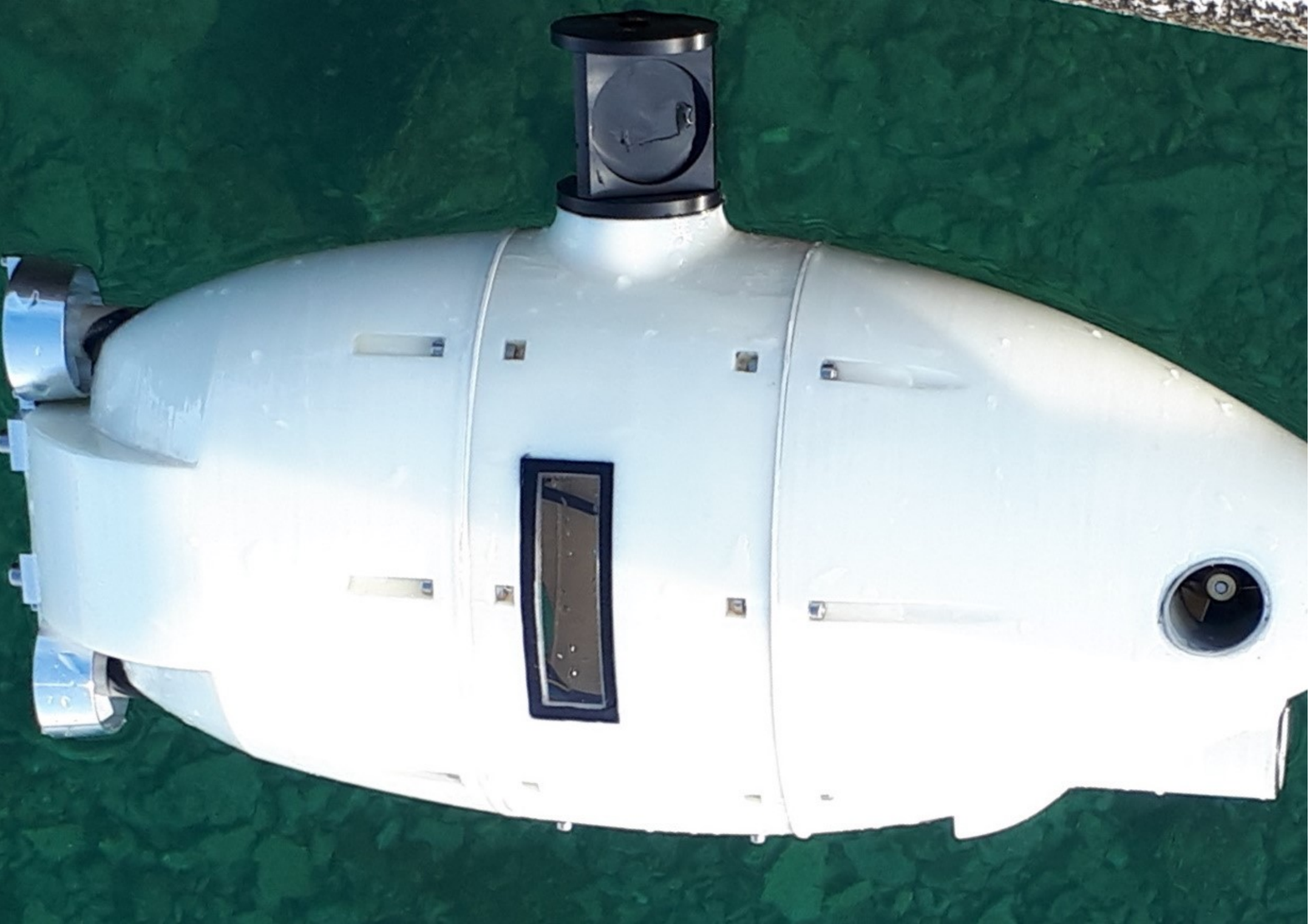}}
    \subfloat{
	\includegraphics[height=0.25\linewidth,trim=270 200 200 200, clip]{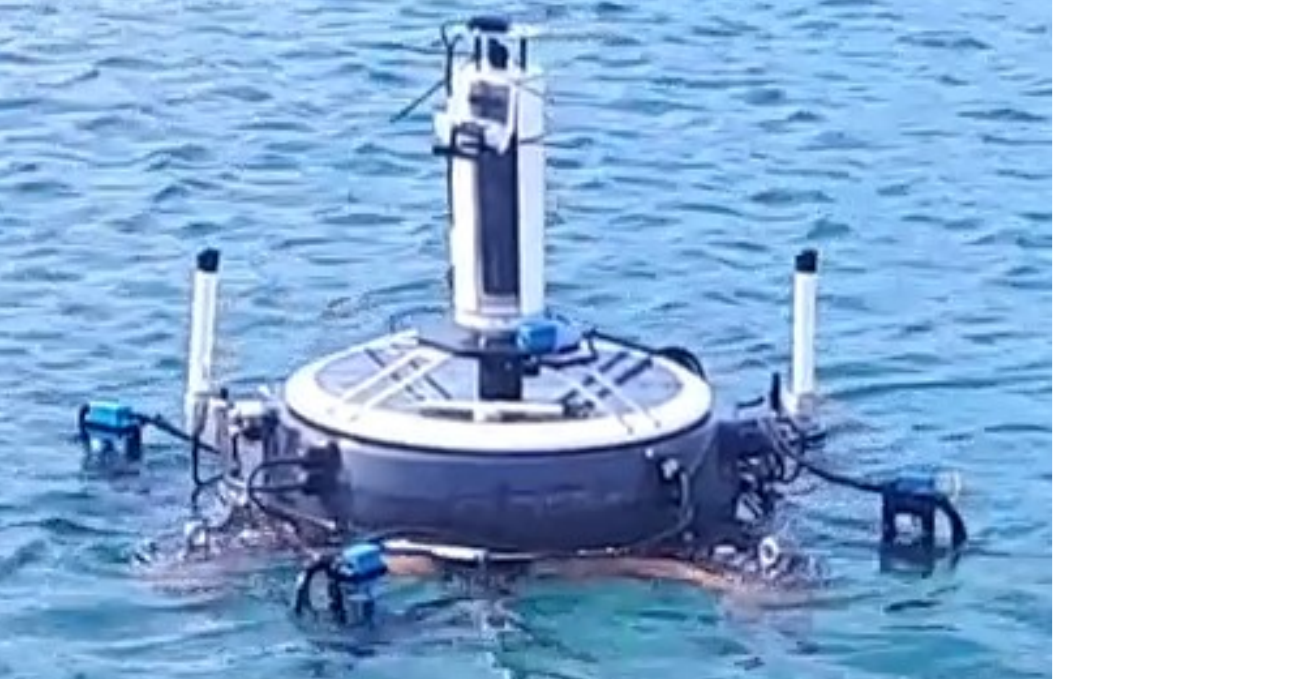}}
    \caption{Robots of the subCULTron swarm -- aMussels (left), aFish (middle) and aPad (right)}
  \label{fig:subCULTron_robots}
\end{figure}
An underwater swarm consisting of a large number of units (100 aMussels, 10 aFish and 5 aPads) was developed within the project.  The ability of such a system to  compute (in a decentralized manner) common estimates of unknown quantities (such as measurements), and agree on a common view of the world, is critical. Consensus protocols (algorithms)  are particularly compelling for implementation in multi-agent systems due to their simplicity and a wide range of applications. Foundation of consensus protocols in multi-agent systems lies in the field of distributed computing. In networks of agents, \emph{consensus} means to "reach an agreement regarding a certain quantity of interest that depends on the state of all agents" \cite{Olfati2007}. A consensus protocol is a series of rules that define information exchange between an agent and all of his neighbors on the network, as well as internal processing of the obtained information by each agent.

A good overview of consensus protocols can be found in \cite{Olfati2007}, \cite{Ren2011}, \cite{Qin2017}. Our primary target is subCULTron system or more precisely, a network of underwater robots measuring environment parameters such as oxygen or turbidity, with sensors prone to faults/errors (outliers). Consensus protocols can be exploited to increase the reliability of such a system by i) developing protocols for detection of faulty agents (such as trust consensus protocols \cite{mazdin2016trust}, \cite{haus2014trust}), or ii) developing consensus protocols that implicitly account for and nullify the influence of outlier values. 

In this paper we apply the later approach - the goal is for each agent to reach a consensus on the real value of the measuring signal. The most common type of consensus problems - average consensus (\cite{Jadbabaie2003}, \cite{ReanBeard2008}) - is not appropriate for this application since even one faulty agent with abnormal values might skew the consensus to incorrect values, which do not represent the real value of the measuring signal. On the other hand, authors in \cite{Franceschelli2014} and \cite{Franceschelli2017} use \textit{median} as the chosen measure, which is inherently robust to outliers, and we follow upon their work in this paper. 

Further, due to the time-varying nature of environment parameters and requirement for spatial distribution of sensors, measurements are performed by a multi-robot system applying the \textit{dynamic} median consensus protocol, where agents track the median of locally available \textit{time-varying} signals, executing local computations and communicating with neighboring agents only \cite{Kia2019}. Dynamic protocol for the average case is available, for both static and varying communication topology (\cite{Zhu2008}).

Work presented in this paper is based upon \cite{Franceschelli2014} and \cite{Franceschelli2017}, where authors present a dynamic median consensus protocol. In this paper we modify their approach so that protocol is functional under scheduled communication protocol, like the one using acoustics for underwater communication. Systems with such a communication scheme are called \emph{switching systems} and their analysis is more complex compared to analysis of static systems. Our previous work \cite{Arbanas20182} and \cite{arbanas2018} deals with consensus using switching systems, but the work presented in this paper reaches the median value rather than the average. Another difference, compared to our previous work, is addition of the dynamic component to the consensus protocol. To conclude, \textbf{the main contribution} of the paper is analysis of novel consensus protocol that is dynamic, works on a switching communication topology and converges to median value, a combination that (to the best of our knowledge) no other papers study. 

There have not been many advancements in the area of underwater consensus protocols. As far as we know, the only underwater consensus applications include formation control of tethered and untethered underwater vehicles \cite{Joordens2010, Putranti2016, Mirzaei2016} and tracking of underwater targets using acoustic sensor networks (ASN) \cite{Yan2017}. Both of those approaches have been validated only in the simulation environment, using ideal communication channels. Among other things, this paper contributes to the field by providing the first experimental application of a consensus protocol in an underwater multi-agent system, acting as a distributed sensor network.

The paper is organized as follows. In the next section, we define some preliminary definitions and notations regarding the systems we study. We present the implemented method for dynamic median consensus over scheduled acoustic communication in Section \ref{sec:consensus}. Simulation results and analysis are presented in Section \ref{sec:simulation}, while the section \ref{sec:experimental} describes the robotic platform we used for experiments and shows the achieved results. Finally, we give a conclusion in Section \ref{sec:conclusion}.
\section{Preliminaries}
\label{sec:prelim}

\subsection{Problem description}
We consider a network of $n$ agents communicating over a single communication channel. The underlying communication graph is defined as a directed graph $\mathbf{G}=(\mathbf{V},\mathbf{E})$, where the set of nodes $\mathbf{V}= \{1,2,..,n\}$ corresponds to agents, and the set of edges $\mathbf{E}\subseteq \mathbf{V} \times \mathbf{V}$ corresponds to communication links between agents. $\mathbf{E}$ is usually described with a corresponding adjacency matrix $\textbf{A}$ such that $a_{ij}=1$ if there is a communication link from agent $j$ to agent $i$, and $a_{ij}=0$ otherwise. 

Each agent $i$ has a local reference signal, denoted
$z_i \in \mathbb{R}, \mathbf{z}=[z_1 z_2 .. z_n]$. In this paper this signal is a measurement made by the agent, but in general it could be the value of some other internal or external variable. Internal state of agent $i$ in step $k$ is denoted as $x_i^k \in \mathbb{R}$. Let us introduce the following notations. 

 A \textit{median value} of vector $\mathbf{z}=[z_1 z_2 ... z_n]$, where elements are listed in ascending order, can be defined as (\cite{Franceschelli2017}):

\begin{equation}
    m(\mathbf{z}) \in
    \begin{cases}
        \{z_{\frac{n+1}{2}}\} & \text{if $n$ is  odd} \\
        [z_{\frac{n}{2}},z_{\frac{n}{2}+1}] & \text{if $n$ is  even} \\
    \end{cases}
    \label{mediandef}
\end{equation}

Agent dynamics in discrete-time domain can  in general be written as:

\begin{equation}
x_i^{k+1} = x_i^k + u_i^k, \quad i \in \{1,2,\ldots,n\}
\label{eq-dynamics}
\end{equation}

where $u_i^k \in  \mathbb{R}$ takes into account both agent's internal  states as well as states of neighbouring agents, which are obtained through communication. We say that the system running (\ref{eq-dynamics}) converges to consensus value $c$ (denoted $x_i^k \xrightarrow{} c$) iff:
 \begin{equation}
    \exists \delta>0,\exists k_0>0, |x_i^k-c|<\delta, \forall i\in \{1,2,\ldots,n\}, \forall k>k_0
    \label{convergence}
\end{equation}

The most common type of such consensus protocols are designed to converge to the \textit{average} value of local reference signals ($z_i$) of all agents. In this paper we deal with \textit{median} consensus protocol, that is, we propose and analyse a protocol such that all the agents converge to the same value, corresponding to the median of their measurements, $x_i^k \xrightarrow{} m(\mathbf{z})$, which means that internal states of all agents $x_i^k$ will converge to the median of their current measurements $\mathbf{z}$.  Further, since the proposed protocol allows individual agents to track a \textit{time-varying}
median of reference signals (measurements), it belongs to a class of \textit{dynamic} consensus protocols.
\begin{table}[]
    \caption{List of variables used}
    \label{tab:variables}
    \centering
    {
    \begin{tabular}{|c|c|}
        \hline
         Variable &  Explanation\\
         \hline\hline
         $x_i^k$& internal state of $i$th agent at step $k$\\ & (communicated over the network)\\
         \hline
         $y_i^k$ & additional internal state of $i$th agent at step $k$ \\ &
         (not communicated over the network)\\
         \hline
         $z_i^k$ & measurement of agent $i$ at step $k$\\
         \hline
         $a_{ij}$ & $1$ if agent $i$ is neighbour of agent $j$, $0$ otherwise\\
         \hline
         $\beta$, $\alpha$, $\gamma$, $\kappa$ & algorithm tuning parameters\\
         \hline
         $n$ & number of agents\\
         \hline
         $r_i$ & number of neighbours of agent $i$\\
         \hline
    \end{tabular}}
\end{table}

\subsection{Communication scheme}
\label{sec:comm}
A communication scheme is such that agents, due to the characteristics of acoustic signals, must take turns transmitting messages. In other words, in order to avoid interference of acoustic signals, one must ensure that agents do not transmit messages in the same time.   
We opt for the simplest solution, which is a sequential assignment of time-slots to agents, repeated cyclically (\emph{round-robin} \cite{OSStallings}). Each time-slot is reserved for sending of messages of one agent. 
\begin{figure}[h!]
    \centering
    \includegraphics[width=\linewidth,trim=4 4 4 4,clip]{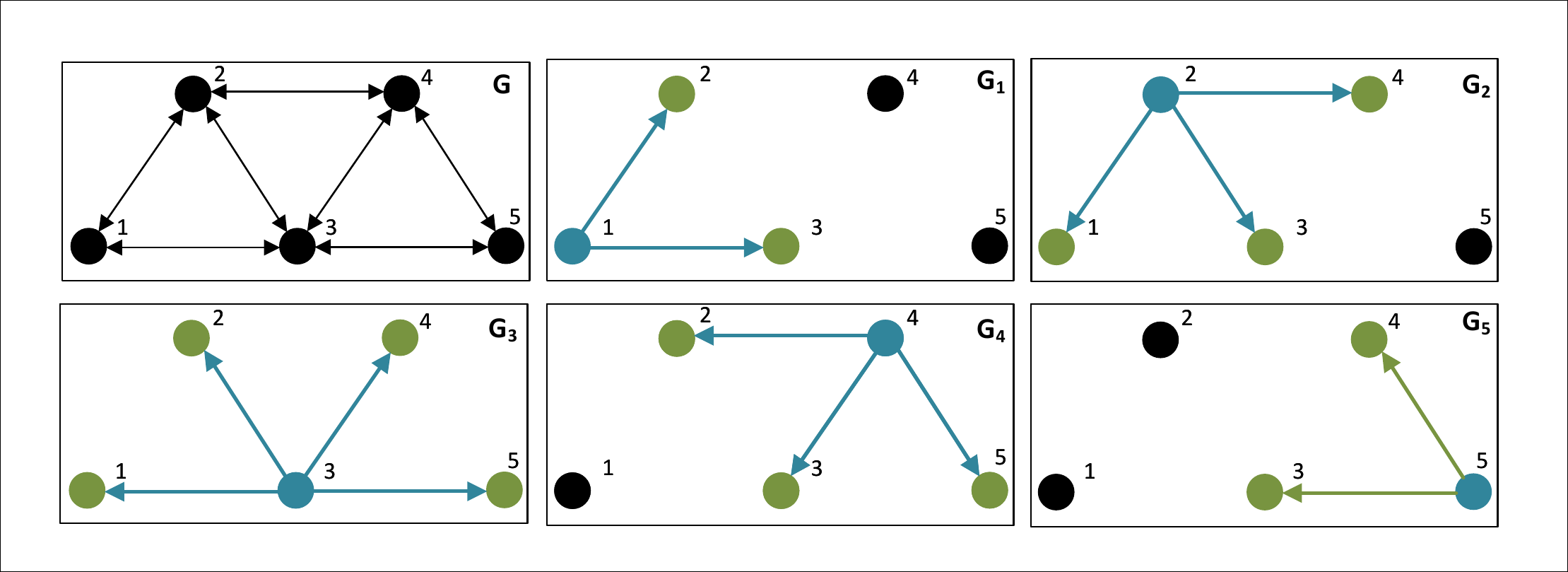}
     \caption{Scheduled acoustic communication - graph $G_i$ for each time-slot and the underlying graph $G$. Sending agent (blue), receiving agents (green).}
    \label{fig:comm_topology}
\end{figure}

The result of the described scheme is a switching communication topology, as given in Figure \ref{fig:comm_topology}. The underlying communication graph $\mathbf{G}$ is an union of communication graphs $\mathbf{G_i}$, $i \in \{1,\ldots, n\}$ over one cycle of described round-robin schedule.  The same communication scheme was used in \cite{Arbanas20182} and \cite{arbanas2018}, where it was more thoroughly elaborated.  Other approaches might be used as well and do not affect the outcome of the proposed protocol.

\section{Dynamic median consensus}
\label{sec:consensus}

We propose the following decentralized protocol, defined using a local update rule of each agent:
\begin{equation}
\begin{split}
    x_i^{k+1}&=x_i^k+a^k_{ij}(\beta(x_j^k-x_i^k)+\frac{\alpha}{r_i} \sign(z_i^k-x_i^k)+y_i^k)\\
    y_i^{k+1}&=y_i^k+a^k_{ij}(\gamma((x_j^k-x_i^k)-\kappa y_i^k))\\
\end{split}
\label{eq:states}
\end{equation}

where $\beta$, $\alpha$, $\gamma$ and $\kappa$ are tuning parameters, $y_i^k$ is an additional internal state of agent $i$ in step $k$, $x_j^k$ is the value of agent $j$ which is transmitting information over communication channel in step $k$ (when $a^k_{ij} \neq 0$) and $r_i$ is total number of neighbours of agent $i$, obtained from graph $\mathbf{G}$ (including agent $i$ itself). List of the variables used is shown in Table \ref{tab:variables}. Initial value is $x_i^0 = z_i^0$. 
 With respect to other median-reaching algorithms (\cite{Franceschelli2014}, \cite{Franceschelli2017}), the novelty is introduction of the internal state variable $y_i$, which is needed to ensure convergence for the underlying switching topology. The signum component of the local rule, characteristic for median-reaching consensuses, is here specified to account for number of neighbor agents ($r_i)$, which is as well needed to ensure stability due to switching nature of the system.

\begin{theorem}
\label{theorem1}
For a local rule of each agent defined as \eqref{eq:states}, and if $y_i \in [-\frac{2\alpha}{r_i},\frac{2\alpha}{r_i}]$, $\gamma<\beta$$ <\frac{1}{n^2}$, and  $\kappa\gamma<1$,
all of the agents converge to the value $m(\mathbf{z})$.
\end{theorem}

\begin{proof}
In order to simplify the equations, let us introduce the following notation 
$\max\limits_{i \in  \mathbf{V}}x_i^k =x^k_{max}$,  
$\min\limits_{i \in  \mathbf{V}}x_i^k =x^k_{min}$,  
Similar as in \cite{Franceschelli2017}, we consider a discrete non-smooth Lyapunov candidate:
\begin{equation}
    V(\mathbf{x^k},\mathbf{y^k})=x^k_{max}-x^k_{min}+y^k_{max}-y^k_{min}
\end{equation}

In order for a system to be stable, the following conditions need to be met:
\begin{equation}
\begin{split}
 V(0)&=0\\
 V(\mathbf{x}^k,\mathbf{y}^k)&>0, \forall \mathbf{x}^k, \mathbf{y}^k \neq \mathbf{0} \\
 dV=V(\mathbf{x}^{k+1},\mathbf{y}^{k+1})-V(\mathbf{x}^k, \mathbf{y}^k )&\leq 0, \forall \mathbf{x}^k, \mathbf{y}^k
\end{split}
\label{eq:lyapunov}
\end{equation}
The first two conditions in \eqref{eq:lyapunov} are always satisfied. The second equation could have value $0$ in case when all $x_i$ and all $y_i$ have the same value. However, this case will never occur, because in equation \eqref{eq:states}, for $x_i$ to be in consensus, all $y_i$ cannot have the same value. The third condition in \eqref{eq:lyapunov} yields:
\begin{equation}
\begin{split}
    dV = &x^{k+1}_{max}-x_{min}^{k+1} +y_{max}^{k+1}-y_{min}^{k+1}-\\
         -&(x^{k}_{max}-x_{min}^{k} +y_{max}^{k}-y_{min}^{k})
\end{split}    
\label{eq:condition}
\end{equation}

By taking into account local interaction rule \eqref{eq:states}, the worst case, from the stability point of view, can be written as:
\begin{equation}
    \begin{split}
        x_{max}^{k+1}&=x^{k}_{max}+\beta(x_j^k-x^{k}_{max})+\frac{\alpha}{r_{min}}+y^{k}_{max}\\
        x_{min}^{k+1}&=x^{k}_{min}+\beta(x_j^k-x^{k}_{min})-\frac{\alpha}{r_{min}}+y^{k}_{min}\\
        y_{max}^{k+1}&=y^{k}_{max}+\gamma((x_j^k-x^{k}_{min})-\kappa y^{k}_{max}))\\
        y_{min}^{k+1}&=y^{k}_{min}+\gamma((x_j^k-x^{k}_{max})-\kappa y^{k}_{min}))\\
    \end{split}
\label{eq:raspisano}
\end{equation}

By substituting \eqref{eq:raspisano} in \eqref{eq:condition}, we obtain the following:
\begin{equation}
      dV=(1-\gamma\kappa)(y^{k}_{max}-y^{k}_{min}) -(\beta-\gamma)(x^{k}_{max}-x^{k}_{min})
    +2\frac{\alpha}{r_{min}}
    \label{eq:skraceno}
\end{equation}

For a system to be asymptotically stable $dV<0$ has to be satisfied. Given $\beta > \gamma$, and, as stated in Theorem \ref{theorem1}, ${y_{max}}$  and $y_{min}$ are limited to $[-\frac{2\alpha}{r_{min}},\frac{2\alpha}{r_{min}}]$, with $\kappa\gamma<1$, we get:

\begin{equation}
    x^{k}_{max}-x^{k}_{min}>\frac{6\alpha-4\kappa\gamma\alpha}{r_{min}(\beta-\gamma)}
    \label{eq:uvjet}
\end{equation}

The condition \eqref{eq:uvjet} is the worst possible case, that is, the right side of the inequality is the lowest theoretical bound on the values of $x^{k}_{max}-x^{k}_{min}$. 

When \eqref{eq:uvjet} is true, the energy in the system will dissipate until the condition in \eqref{eq:uvjet} becomes false. In that moment  $x_{max}-x_{min}$ will start increasing until equation \eqref{eq:uvjet} becomes true.  By getting closer to the consensus, $y_i^k$ will converge to either $\frac{\alpha_i}{r_i}$ or $-\frac{\alpha_i}{r_i}$.

Hence, the system is stable under the given condition, however, the system \eqref{eq:states} does not reach a steady state in the classical sense. We assume that the steady state is reached when the values over one communication cycle of length $n$ become stationary, that is, when $x_i^{k+n}=x_i^k, y_i^{k+n}=y_i^k$, $\forall i$, $\forall k > k_0$. From \eqref{eq:states} we get: 

\begin{equation} 
\label{eq:steadystate}
\begin{split}
x_i^{k+n}=x_i^k  &+\beta\sum_{j=1}^{n}{a_{ij}(x_j^{k+j}-x_i^{k+j})}
                      +\sum_{j=1}^{n}{a_{ij}y_i^{k+j}}\\ &+\frac{\alpha}{r_i}\sum_{j=1}^{n}{a_{ij}\sign(z_i-x_i^{k+j})}\\
y_i^{k+n}=y_i^k&+\gamma\sum_{j=1}^{n}{a_{ij}(x_j^{k+j}-x_i^{k+j})}               -\gamma \kappa \sum_{j=1}^{n}{a_{ij}y_i^{k+j}}
\end{split}
\end{equation}

By summing the expressions for $y_i^{k+n}$ in \eqref{eq:steadystate} for all agents $i$  in stationary state we get:

\begin{equation}
\begin{split}
\sum_{i=1}^n\sum_{j=1}^{n}{a_{ij}(x_j^{k+j}-x_i^{k+j})} =\\
\kappa \sum_{i=1}^n \frac{-\alpha}{r_i(1+\beta\kappa)} \sum_{j=1}^{n}{a_{ij}\sign(z_i-x_i^{k+j})}
\end{split}
\label{eq:sumsteadystate}
\end{equation}

In equation \eqref{eq:sumsteadystate} if we assume that during $n$ steps, the value of $\sign(z_i-x_i^{k+j})$ doesn't change for any $i$, the following must hold:

\begin{equation}
\begin{split}
\sum_{i=1}^n\sum_{j=1}^{n}{a_{ij}(x_j^{k+j}-x_i^{k+j})} =
\frac{-\alpha\kappa}{1+\beta\kappa}\sum_{i=1}^n {\sign(z_i-x_i^k)}
\end{split}
\label{eq:sumsteadystate2}
\end{equation}

Under the assumption that the system reaches consensus of all agents, the following holds:

\begin{equation}
\sum_{i=1}^{n} \sum_{j=1}^{n}{a_{ij}(x_j^{k+j}-x_i^{k+j})}\approx 0
    \label{eq:assumption}
\end{equation}
which in turn gives:

\begin{equation}
    \sum_{j=1}^n {\sign(z_i-x_i^{k+j})} \approx0
    \label{eq:sumsignum}
\end{equation}

In the case when all agents have exactly the same value, $c=x_i^k$ $\forall i$, it is clear from \eqref{eq:sumsignum} that $c$ must be the median value of vector $\mathbf{z}$. However, according to the definition of consensus \eqref{convergence}, not all $x_i^k$ need to have exactly the same values. For that reason there is a need for further analysis.

Without the loss of generality we can sort the elements under the sum in the equation \eqref{eq:sumsignum} in the ascending order of their elements $z_i$:

\begin{equation}
    \sum_{i=1}^n {\sign(z_{l(i)}-x_{l(i)}^k)} \approx0
    \label{eq:sortsumsignum}
\end{equation}
where $l(i)$ rearranges indices from \eqref{eq:sumsignum} in such a way that $z_{l(i)}<z_{l(i+1)},\forall i$. It is important to note that each element in the sum in equation \eqref{eq:sumsignum} has corresponding element in the sum in the equation \eqref{eq:sortsumsignum}.

After that we can split the sum into two sums:
\begin{equation}
    \sum_{i=1}^{n/2}\sign(z_{l(i)}-c+\delta)+\sum_{i=n/2+1}^{n}\sign(z_{l(i)}-c-\delta)\approx 0
    \label{eq:splitsumsign2}
\end{equation}
where $x_{l(i)}^k\in [c-\delta,c+\delta]$, as defined in \eqref{convergence}.

The exact value of $c$ depends on the values in vector $\mathbf{z}$, but it is clear that for
\begin{equation}
    c \in [m(z)-\delta,m(z)+\delta]
    \label{eq:rangec}
\end{equation}

the system will converge to the median values of measurements $\mathbf{z}$.

Finally, since matrix $\mathbf{A}$ is symmetrical, every element $(x_j^{k+j_1}-x_i^{k+j_1})$ in the sum in \eqref{eq:sumsteadystate} has its pair element $(x_i^{k+j_2}-x_j^{k+j_2})$, except elements in the diagonal which by default have value 0. Hence, from equations \eqref{eq:steadystate} and \eqref{eq:sumsteadystate}, one can determine $\beta$ from the worst case scenario:
\begin{equation}
    \frac{1}{2}\sum_{i=1}^n\sum_{j=1}^n \frac{\beta \kappa \alpha- (-\beta \kappa \alpha)}{1+\beta \kappa}<\frac{\kappa\alpha}{1+\beta\kappa}
\end{equation}

which leads to:
\begin{equation}
    \beta<\frac{1}{n^2}
\end{equation}

\end{proof}

\section{Simulation results}
\label{sec:simulation}

\subsection{Convergence Analysis}
In this section we present simulation results of the proposed dynamic median consensus protocol. We conducted simulations for system with 3, 5 and 31 agents. In the experimental section, we conducted tests with 3 and 5 agents. A more extensive network of 31 agents is chosen to illustrate the algorithm's performance on a larger scale. For each number of agents, two different communication topologies were used, a complete and chain topology, as given in Figure \ref{fig:sim-E}. These topologies correspond to the best and worst case, since in order to propagate information from agent $i$ to agent $j$ we need only one step for Figure \ref{fig:sim-E1} (best case) and $n$ steps for Figure $\ref{fig:sim-E2}$ (worst case). Graphs given in this figure correspond to the overall communication topology, while topology in each step behaves as given in subsection \ref{sec:comm}.
\begin{figure}[h]
\centering
    \subfloat[Complete topology]{
	    \includegraphics[width=0.33\linewidth, trim=25 25 25 25, clip]{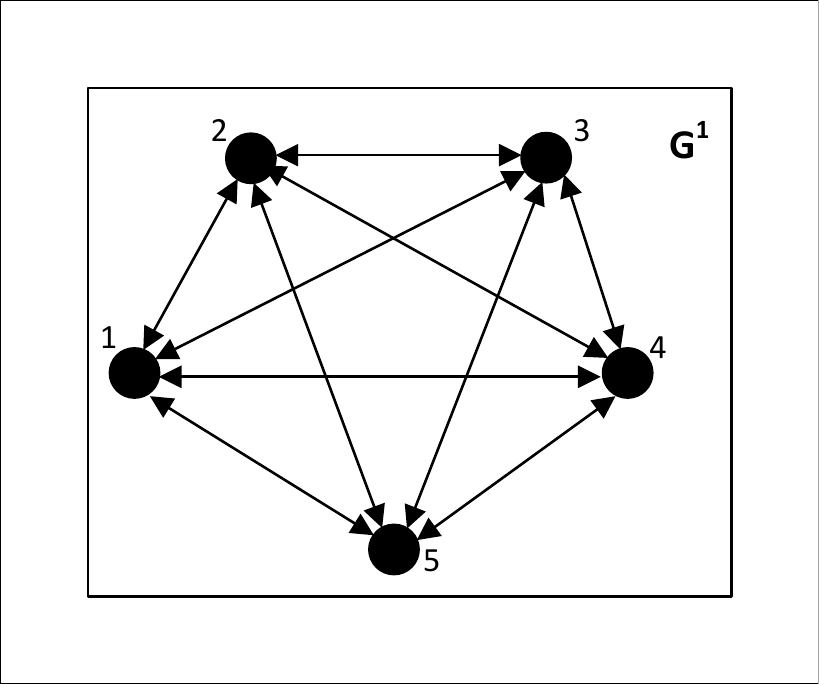}
        \label{fig:sim-E1}
    }
    \subfloat[Chain topology]{
	    \includegraphics[width=0.6\linewidth, trim=25 25 25 25, clip]	{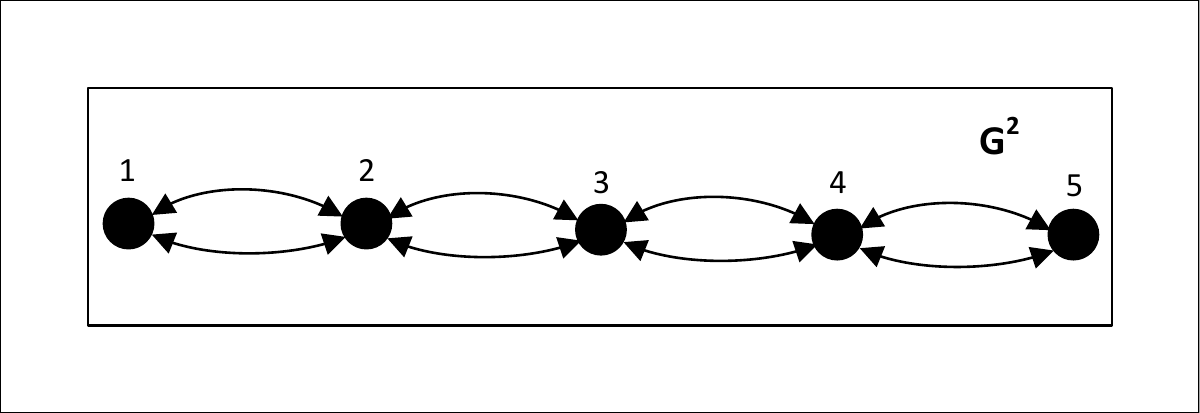}
        \label{fig:sim-E2}
    }
    \caption{Communication topologies used in simulation}
    \label{fig:sim-E}
\end{figure}

Adjacency matrices for the above topologies are:
\begin{gather*}
    \mathbf{A^1}=
    \begin{bmatrix}
    1 & 1 & ... &  1\\
    1 & 1 & ... &  1\\
    ... & ... & ... &...\\
    1 & 1 & ... &  1\\
    \end{bmatrix}
    \\
    \mathbf{A^2}=
    \begin{bmatrix}
    1 & 1 & 0 & 0 & 0 & ... & 0 & 0\\
    1 & 1 & 1 & 0 & 0 & ... & 0 & 0\\
    0 & 1 & 1 & 1 & 0 & ... & 0 & 0\\
    ... & ... & ... & ... & ...& ... & ... & ...\\
    0 & 0 & 0 & 0 & & ... & 1 & 1\\
    \end{bmatrix}
\end{gather*}

Table \ref{tab:parametri} shows the number of agents, consensus tuning parameters (as given in Eq. (\ref{eq:states})) and results, for six different simulation setups. The numerical results are presented  with 3 parameters: $t_s$ settling time (the number of steps needed for all agents to reach value that is within 5\% of measurements median value), $t_c$ convergence time (the number of steps needed for maximal distance between two agents to reach value that is within 5\% of measurements median value) and $\epsilon_{ss}$ error in steady state (value of maximal distance of agents from measurements median value with respect to the median value)
Influence of each of the tuning parameters on the system behaviour is analysed later in this section.

\begin{table}[h!]
\centering
\caption{Tuning parameters and results in each simulation run}\label{tab:parametri}
	\begin{tabular}{|c|c|c|c|c|c|c|}
    \hline
    	Parameter & Sim 1 & Sim 2 & Sim 3 & Sim 4 & Sim 5 & Sim 6\\
        \hline
        \hline        
        $n$ & 3 & 3 & 5 & 5 & 31 & 31 \\
        \hline
        $\alpha$ & 9 & 9 & 3 & 3 & 2 & 2\\
        \hline
        $\beta$ & 0.08 & 0.08 & 0.04 & 0.04 & 0.001 & 0.001\\
        \hline
         $\gamma$ & 0.003 & 0.003 & 0.0015 & 0.0015 & 0.0005 & 0.0005\\
         \hline
         $\kappa$ & 0.1 & 0.1 & 0.1 & 0.1 & 0.1 & 0.1\\
         \hline 
         $\mathbf{A}$ &
         $\mathbf{A^1}$ & $\mathbf{A^2}$ & $\mathbf{A^1}$ & $\mathbf{A^2}$ & $\mathbf{A^1}$ & $\mathbf{A^2}$\\
        \hline
        $t_s$ & 51 & 238 & 110 & 729 & 26856 & 342520 \\
        \hline
        $t_c$ & 53 & 230 & 111& 753 & 11811 & 440231 \\
        \hline
        $\epsilon_{ss}$ & 2.25\% & 2.46\% & 0.39\% & 0.63\%& 0.04\% & 1.6\% \\
        \hline
    \end{tabular}
\end{table}

We set an initial measured value for each agent, and then make a step change of individual measurements so that the overall median changes stepwise. Results for 3 agents are shown in Figures \ref{fig:n3e1} and \ref{fig:n3e2}. Figures show responses of individual agent values $\mathbf{x}$ and their additional internal states $\mathbf{y}$ for different communication topologies. Results  for 5 and 31 agents are given in Figures \ref{fig:n5e1} - \ref{fig:n31e2}. Results show that, after a transient period, the system converges to correct values of median, and is able to dynamically track changes in median values. Convergence time is almost seven times shorter for a complete topology, than for a chain topology. Stationary values of y for each agent $i$ are $\pm\frac{\alpha}{r_i}$, as given in Eq. \eqref{eq:states}.

\begin{figure}[h]
    \centering
	\includegraphics[ width=0.75\linewidth,trim=80 255 80 260, clip]{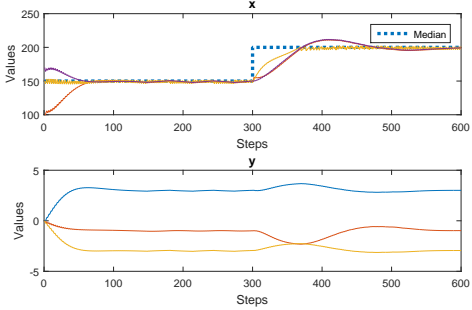}
    \caption{Simulation results for 3 agents (complete topology)}
    \label{fig:n3e1}
\end{figure}

\begin{figure}[h]
    \centering
	\includegraphics[width=0.75\linewidth,trim=80 255 80 260, clip]	{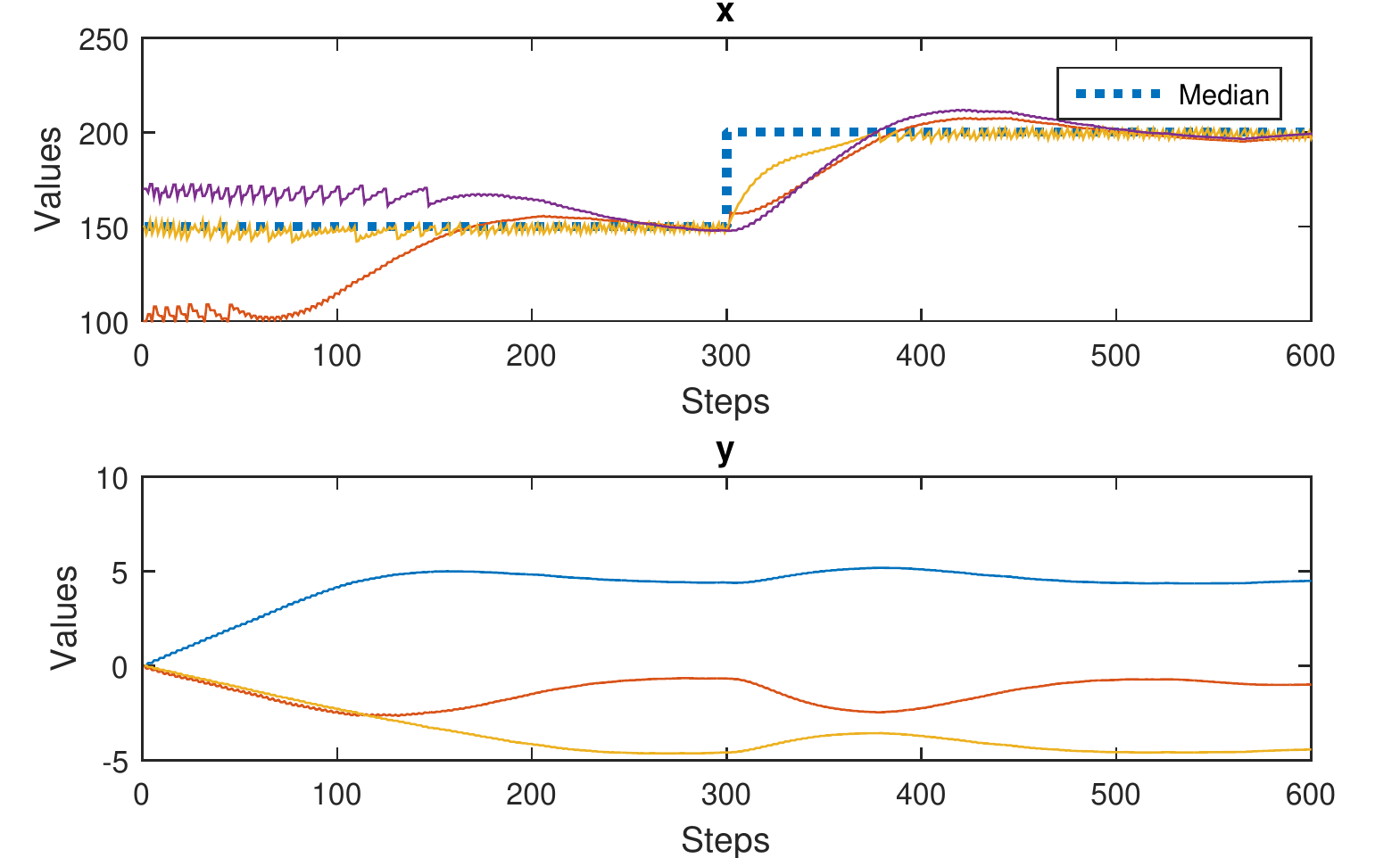}
    \caption{Simulation results for 3 agents (chain topology)}
    \label{fig:n3e2}
\end{figure}

\begin{figure}[h!]
    \centering
	\includegraphics[width=0.75\linewidth,trim=80 255 80 260, clip]{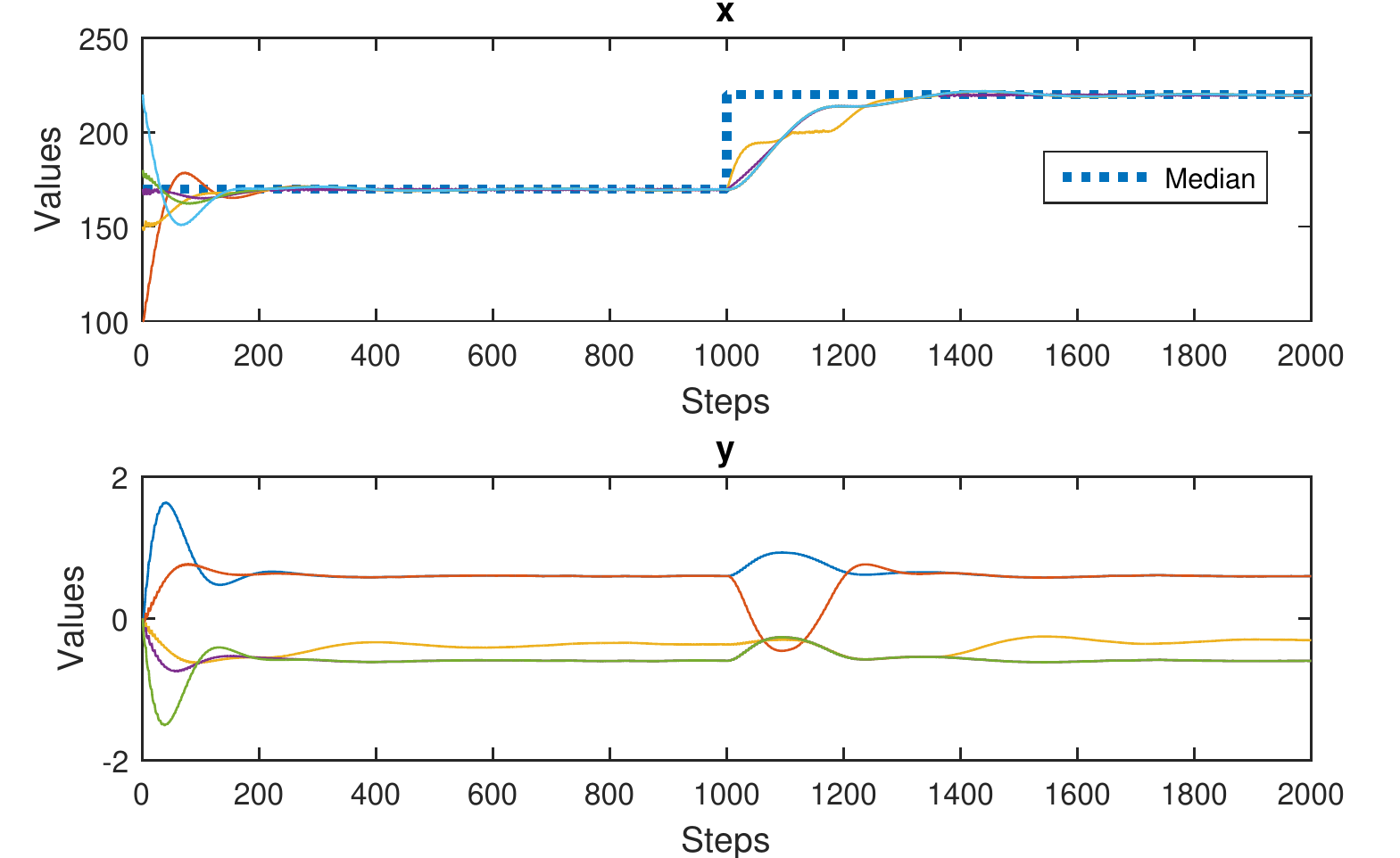}
    \caption{Simulation results for 5 agents (complete topology) }
    \label{fig:n5e1}
\end{figure}

\begin{figure}[h!]
    \centering
    \includegraphics[width=0.75\linewidth,trim=80 255 80 260, clip]{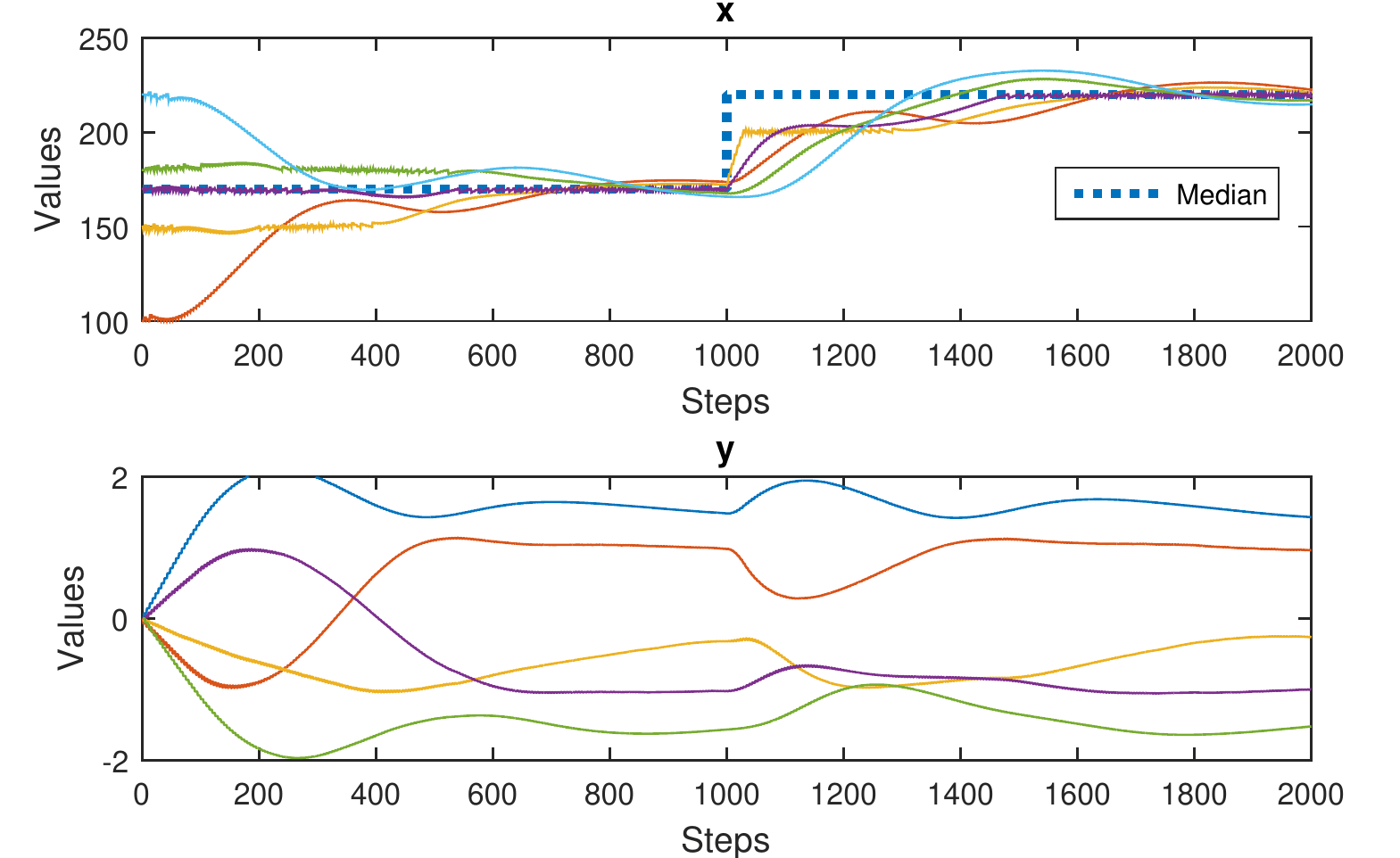}
    \caption{Simulation results for 5 agents (chain topology) }
    \label{fig:n5e2}
\end{figure}

\begin{figure}[h!]
    \centering
	\includegraphics[width=0.75\linewidth,trim=80 285 80 290, clip]{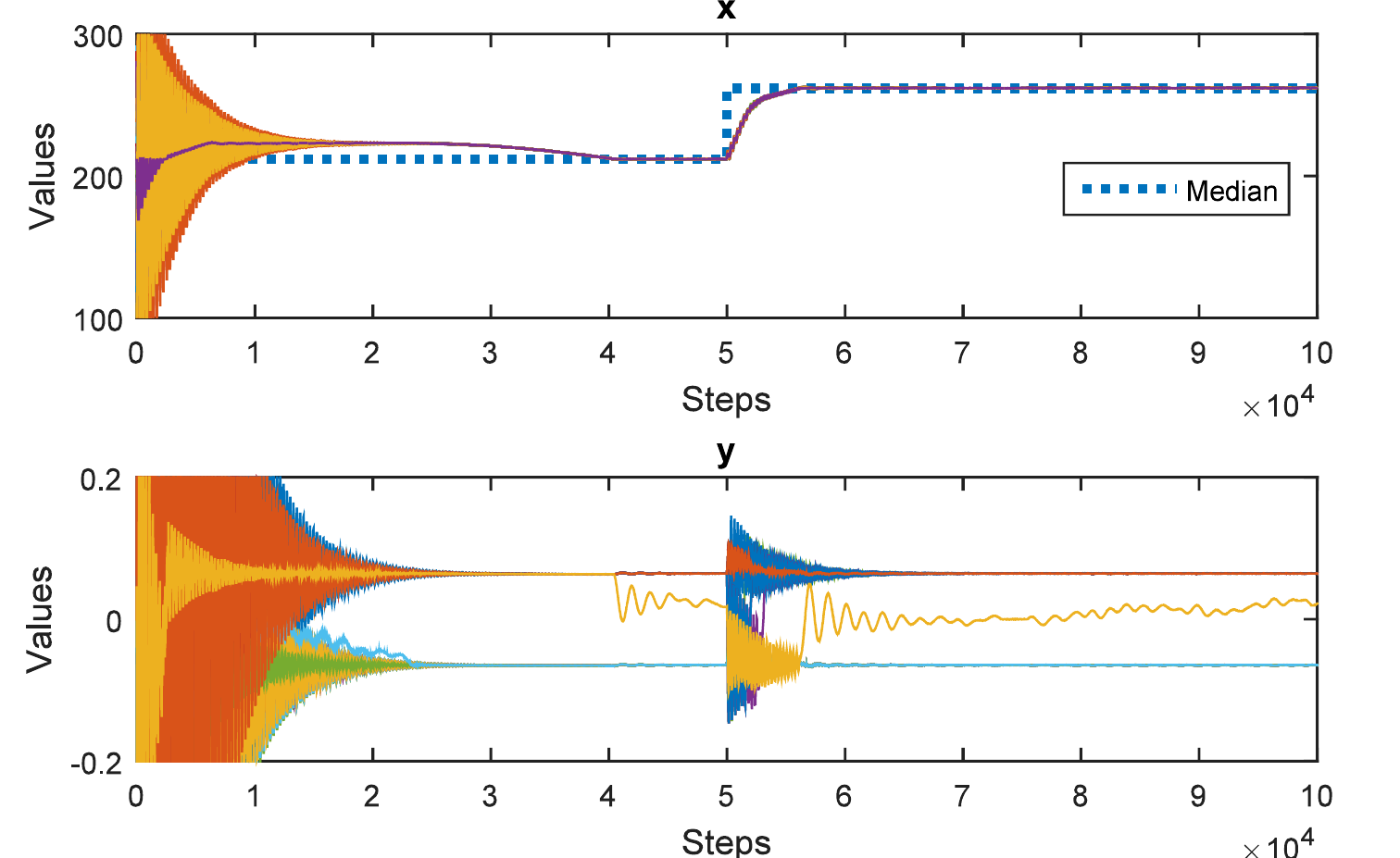}
    \caption{Simulation results for 31 agents (complete topology)}
    \label{fig:n31e1}
\end{figure}

\begin{figure}[h!]
\centering
	    \includegraphics[width=0.75\linewidth,trim=80 285 80 290, clip]	{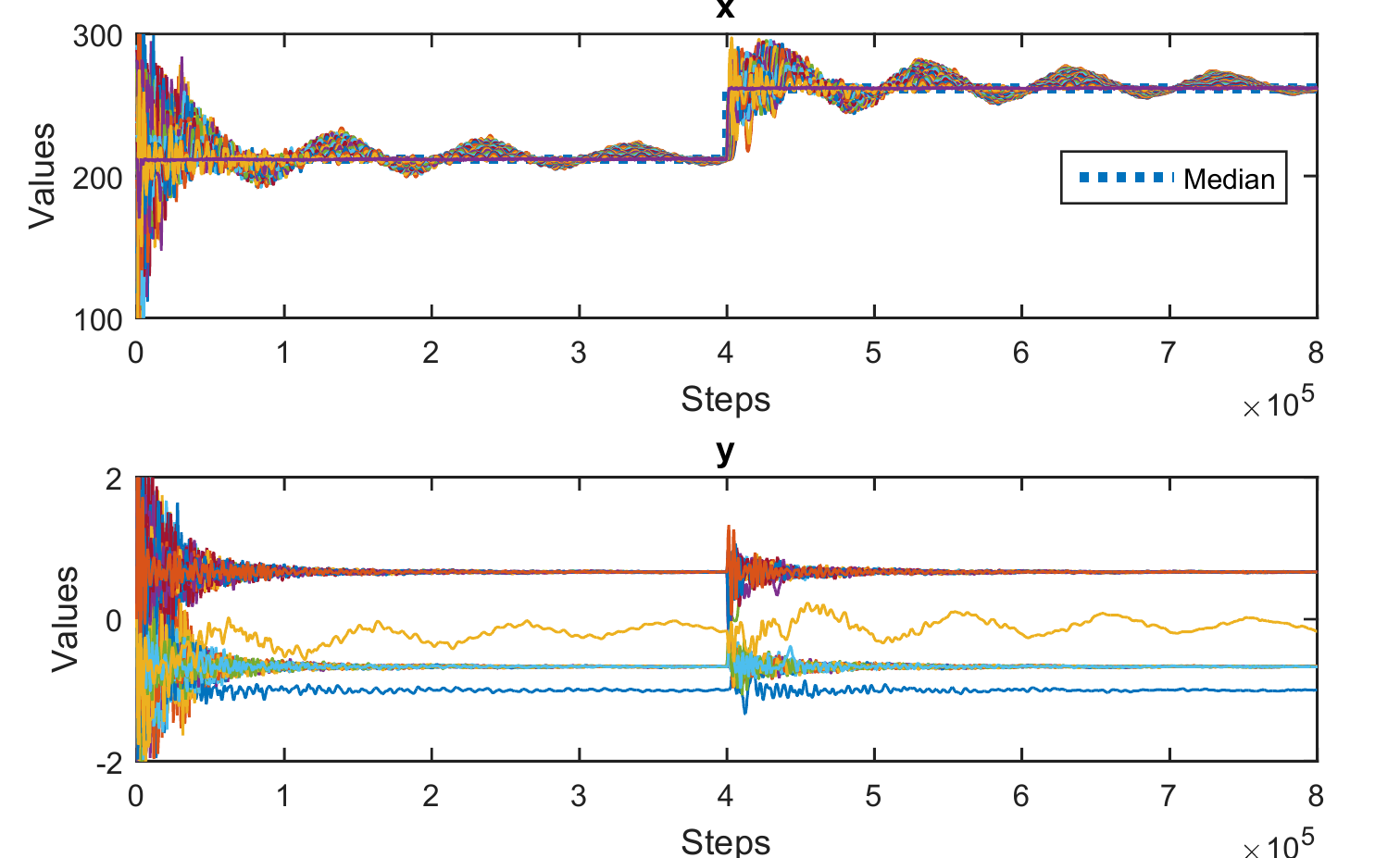}
    \caption{Simulation results for 31 agents (chain topology)}
     \label{fig:n31e2}
\end{figure}

Results for higher number of agents show that settling (convergence) time increases for higher number of agents. Range of values around the median will also influence the convergence time. In general, larger range leads to slower convergence, which is a consequence of the existence of the sign part in the equation for local interaction rule. 

The method was also tested for the case of the lost communication packages. Table \ref{tab:lostcomm} shows influence of random loss of transmitted packages on settling time $t_s$. Results were obtained by simulation for 5 agents with complete topology. For each percentage of lost packages, 100 simulations were run and the table presents mean settling time $t_{savg}$ and its standard deviation $\sigma(t_{s})$, as well as mean steady state error $\epsilon_{avg}$ and its standard deviation $\sigma(\epsilon)$. The results show that the system converges to the median of measured values, but it takes longer time with a higher percentage of lost messages. Simulations have shown that even with higher percentage of lost communication packages, the system converges towards the median of the measurements. Simulations have also shown that lost messages have no influence on the steady state error.

\begin{table}[]
    \centering
    \caption{Influence of the random loss of transmitted packages on the settling time, tested on 100 simulation for the case of 5 agents with complete topology}
    \begin{tabular}{|c|c|c|c|c|}
        \hline
         Lost packages & $t_{savg}$  & $\sigma(t_s)$ & $\epsilon_{avg}$ & $\sigma(\epsilon)$ \\
         \hline\hline
        0\% &	110 &	0  & 0.43\% & 0\%\\\hline
10\%	 & 964	& 1942 & 0.39\% & 0.035\%\\\hline
20\%	& 1527 &	2627 & 0.38\% & 0.037\%\\\hline
30\% &	1982 &	3303 & 0.38\% & 0.034\%\\\hline
40\%  &	2773 &	3831 & 0.37\% & 0.038\%\\\hline
50\% &	4587 &	5508 & 0.38\% & 0.041\%\\\hline
    \end{tabular}
    \label{tab:lostcomm}
\end{table}

System was further tested for sine-shaped measurements of each agent, representing a more dynamical measurement signal. Results ($x$, $y$), together with measurements for each of the 5 agents ($z$), are shown in Figure \ref{fig:sin}. After a transient period the system reaches consensus on the median value and is able to continuously track the median value of all measurements. The second half of the response shows the case of faster changing sinus function. The results show that the system is still able to track the median, but with a delay.

Results are given for an odd number of agents. For an even number, the convergence value is in the interval defined with \eqref{convergence}, and exact value depends on the protocol parameters and values of measured signals. The presented method works for both odd and even number of agents. However, results were presented only for an odd number of agents, since in this case the median is a single value, as opposed to system with even number of agents where median is a range of values.

\begin{figure}[h]
\centering
	    \includegraphics[width=0.65\linewidth,trim=110 230 110 230, clip]	{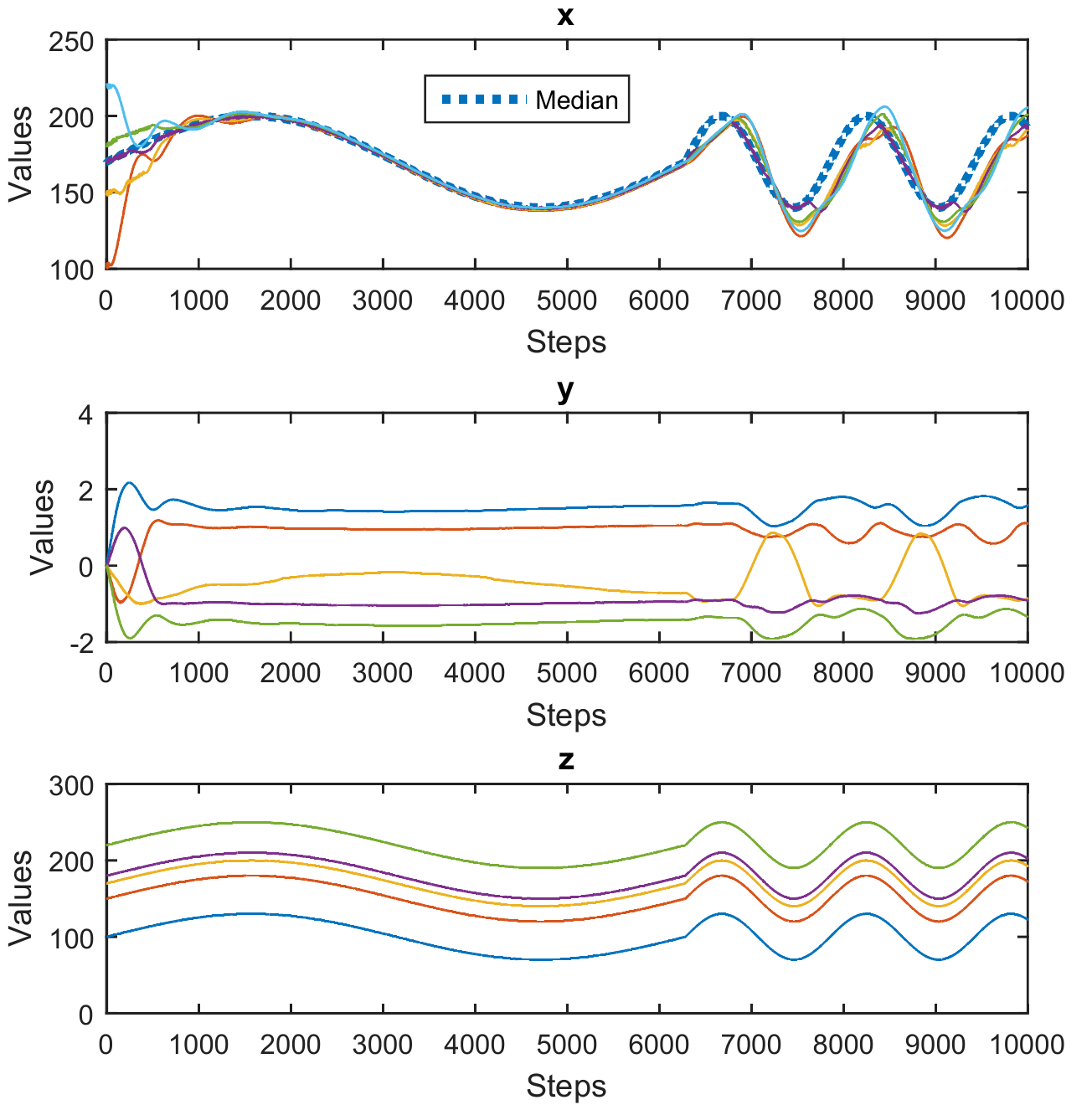}
    \caption{Simulation results for sine measurements}
     \label{fig:sin}
\end{figure}

\subsection{Analysis of tuning parameters}
Each of the algorithm tuning parameters in Table \ref{tab:variables} influences the behaviour of the system:
\begin{itemize}
    \item Increase of parameter $\alpha$ causes faster convergence, but according to the equation \eqref{eq:uvjet} increases the area of instability around the median value. 
    \item Increase of parameter $\beta$ shortens the time needed for agents to reach a common value (which initially does not have to correspond to the median).  
    \item With parameter $\gamma$  we influence the rate of change of the state $\mathbf{y}$, which then influences the convergence speed of $\mathbf{x}$. At the same time, its increase towards $\beta$ increases the instability of the system.
    \item $\kappa$ is the forgetting factor of the state $\mathbf{y}$, too large values cause instability of the system, and with too small values, system converges to the wrong value.
\end{itemize}
The influence of parameters on system response is shown in Figure \ref{fig:tuningparams}, which shows two graphs for each parameter change: i) initial start of response, which presents  the way the agents converge towards the median value and ii) steady state response, which presents the influence of the parameter to the steady state. Numerical results are shown in Table \ref{tab:tuning}. These results were obtained with 5 agents, using tuning parameters Sim 4 from Table \ref{tab:parametri}, where only one parameter changed its value.

\begin{table}
    \centering
    \caption{Influence of tuning parameters to responses}
    \begin{tabular}{|c|c|c|c|}
    \hline
         Parameter & $t_c$ & $t_s$ & $\epsilon_{ss}$ \\
         \hline
         \hline
         $\alpha=1$ & 1082 & 949 & 0.36\%\\
         \hline 
         $\alpha=3$ & 753 & 729 & 1.12\%\\
         \hline
         $\beta=0.01$ & 2659 & 1669 & 1.07\%\\
         \hline
         $\beta=0.08$ & 660 & 463 & 1.12\%\\
         \hline
         $\gamma=0.0015$ & 753 & 729 & 1.12\%\\
         \hline
         $\gamma=0.01$ & 1415 & 1044 & 1.12\%\\
         \hline
         $\kappa=0.02$ & 1047 & 679 & 0.69\%\\
         \hline
         $\kappa=0.4$ & 1639 & 794 & 4.12\%\\
        \hline
    \end{tabular}
    \label{tab:tuning}
\end{table}

\begin{figure}[h]
    \centering
    \includegraphics[width=0.49\linewidth,trim=50 230 50 250, clip]	{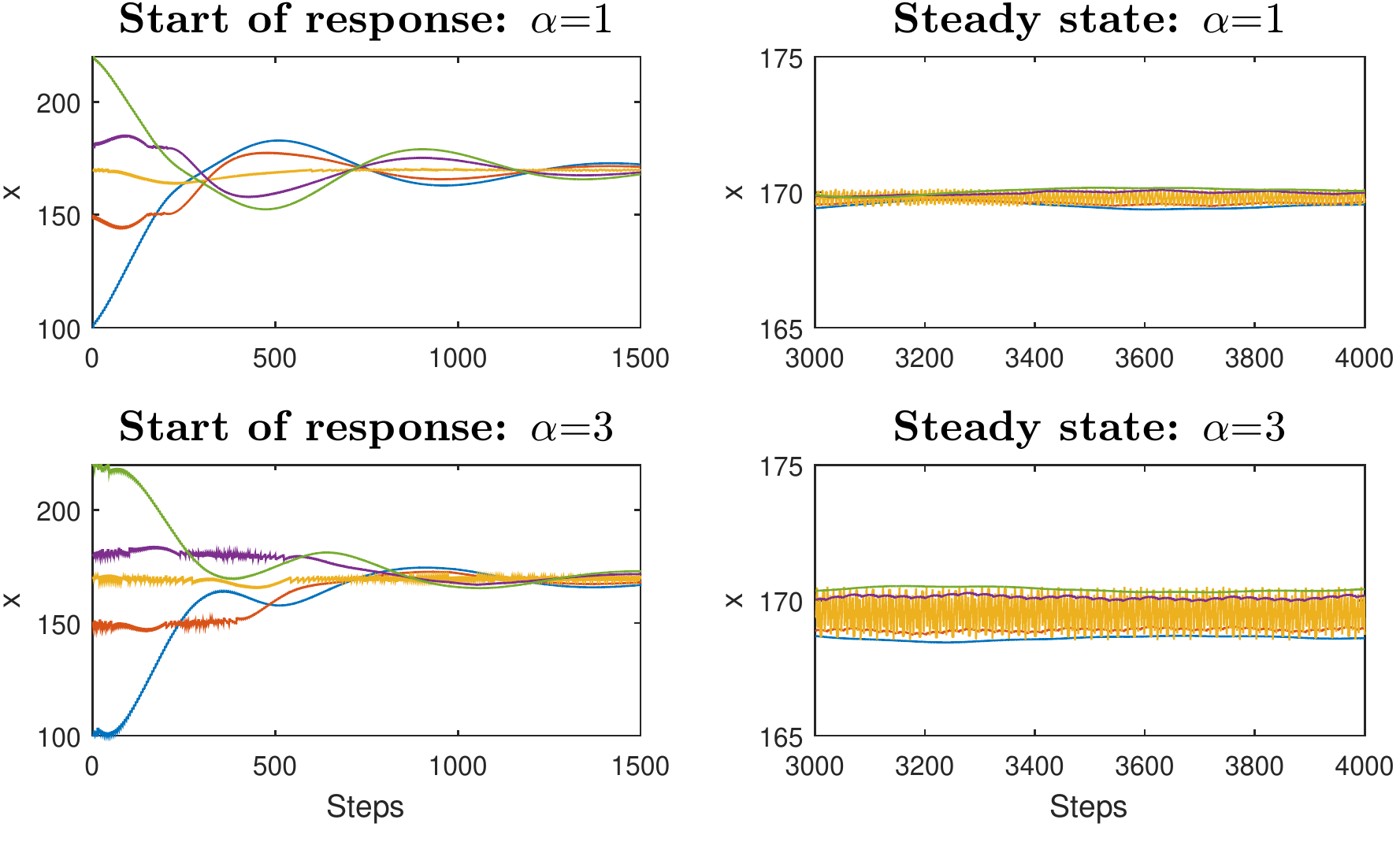}
    \includegraphics[width=0.49\linewidth,trim=50 230 50 250, clip]	{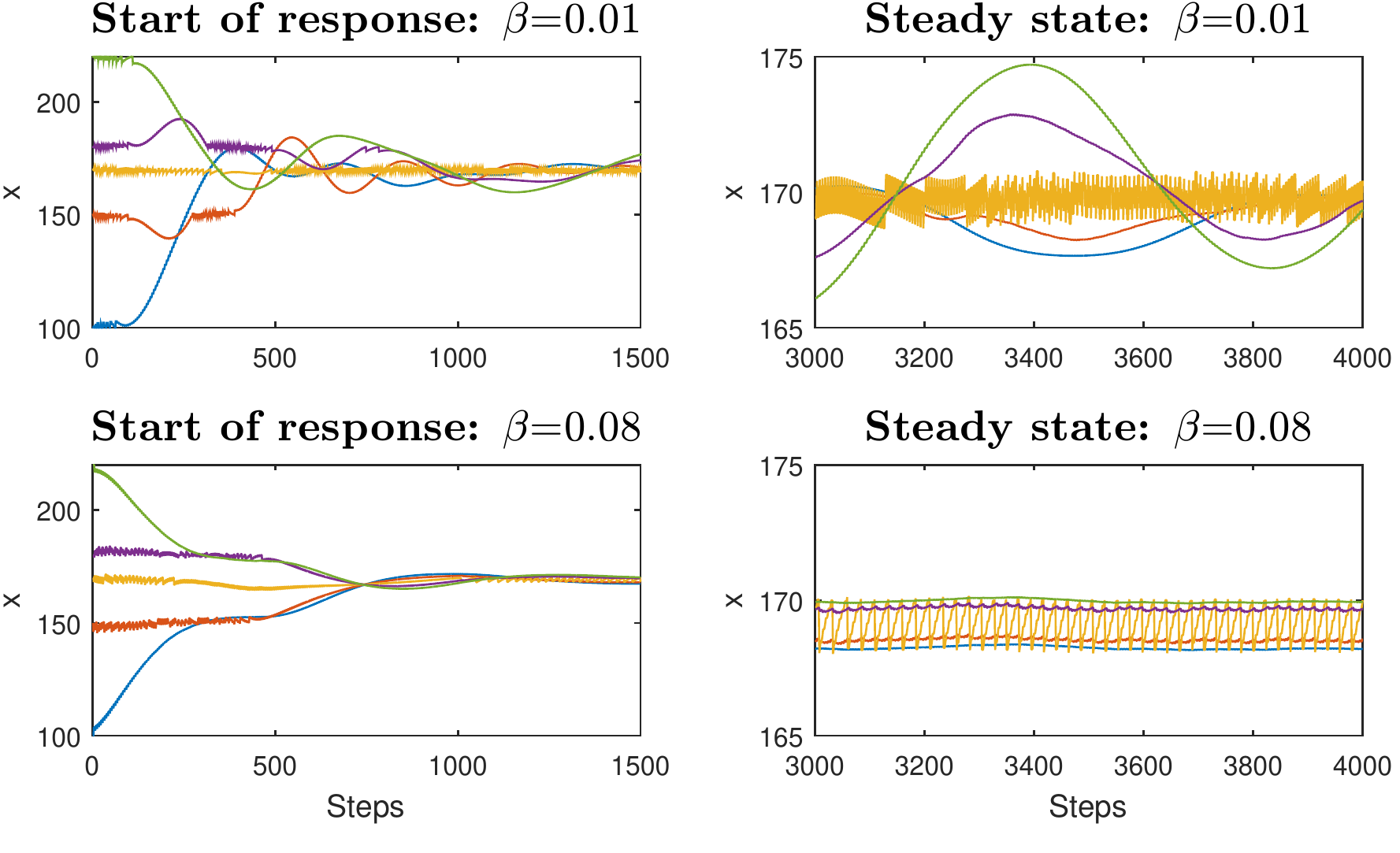}
    \includegraphics[width=0.49\linewidth,trim=50 230 50 250, clip]	{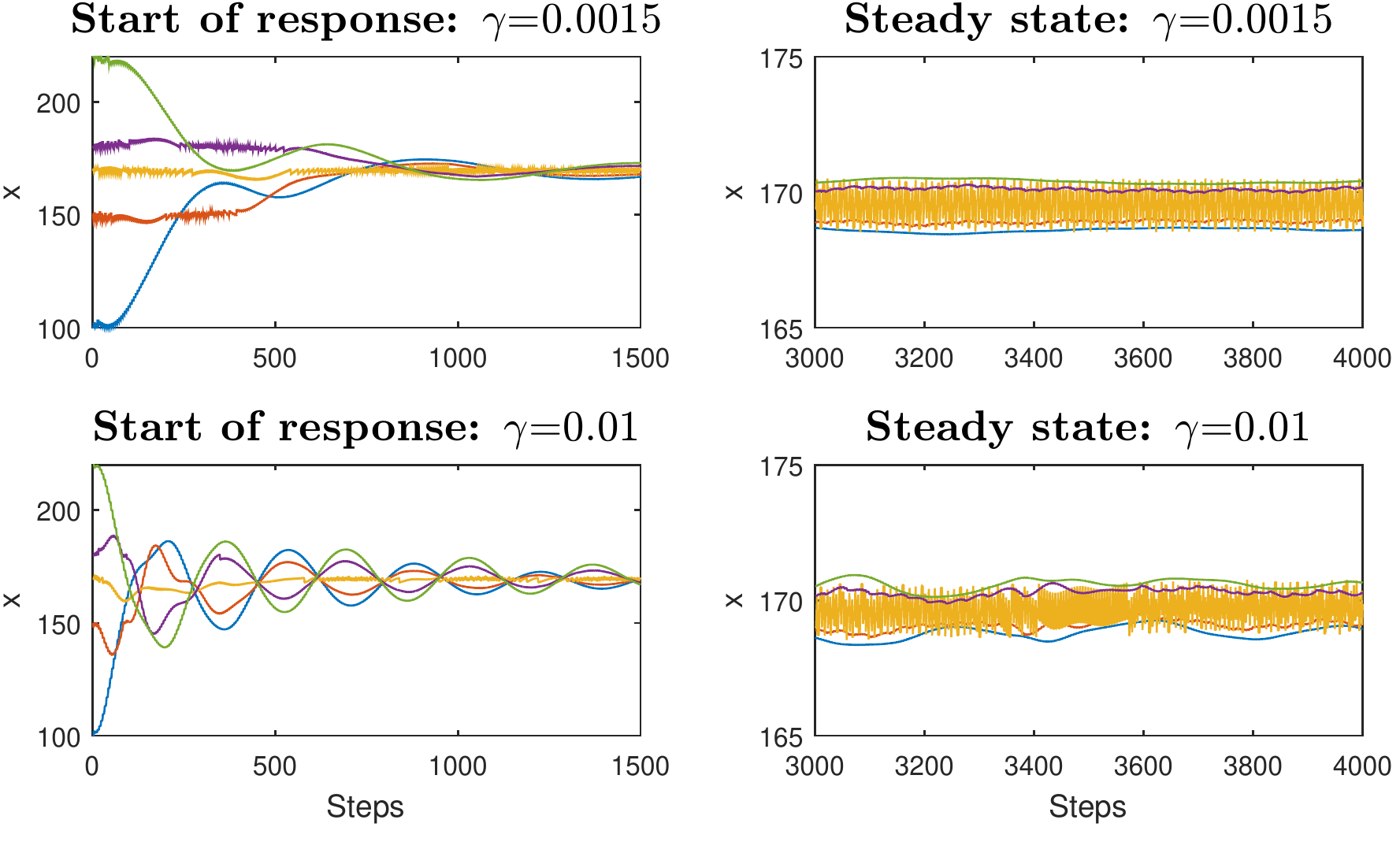}
    \includegraphics[width=0.49\linewidth,trim=50 230 50 250, clip]	{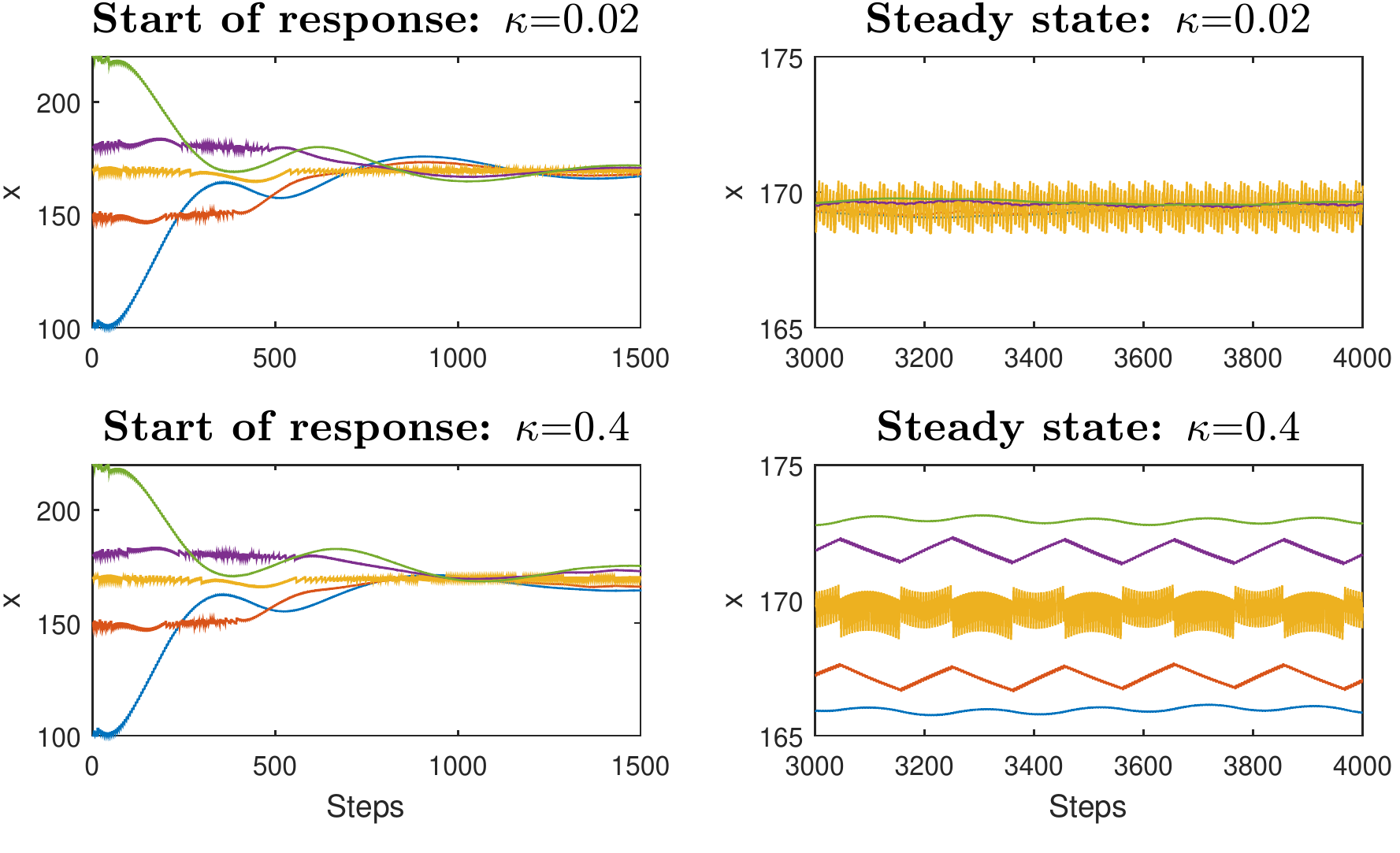}
 \caption{Influence of tuning parameters to response}
    \label{fig:tuningparams}
\end{figure}

\section{Experimental results}
\label{sec:experimental}

In this section we present the results of experiments performed with aMussels (Fig. \ref{fig:amussel}). The aMussel platform (described in \cite{Loncar2019}) is a robotic unit intended for long-term monitoring of underwater areas. It is designed for low power operations on sea-bed, where it takes measurements in regular intervals and goes to low power mode in between. It is equipped with various sensors used to perceive the environment, and two underwater communication devices to share acquired data and control signals with other robots. It is capable only of 1D movement, as it can change its depth using simple yet effective buoyancy system. The aMussel is powered from two single-cell LiPo batteries, where each of them powers different elements of the system. As a communication channel between agents (aMussels) in presented experiments, we have used a nanomodem acoustic unit. For a successful communication, only one of the aMussels in the group is allowed to transmit data using a nanomodem at single moment. For these reasons we have developed a time scheduling scheme, where each of the aMussels has an assigned timeslot in which it is allowed to transmit information using acoustics. Detail about time-scheduling implementation are described in \cite{Loncar2019} and in section \ref{sec:comm}.

\begin{figure}[t]
	\centering
	\includegraphics[width=0.55\linewidth,trim=0 90 0 90, clip]{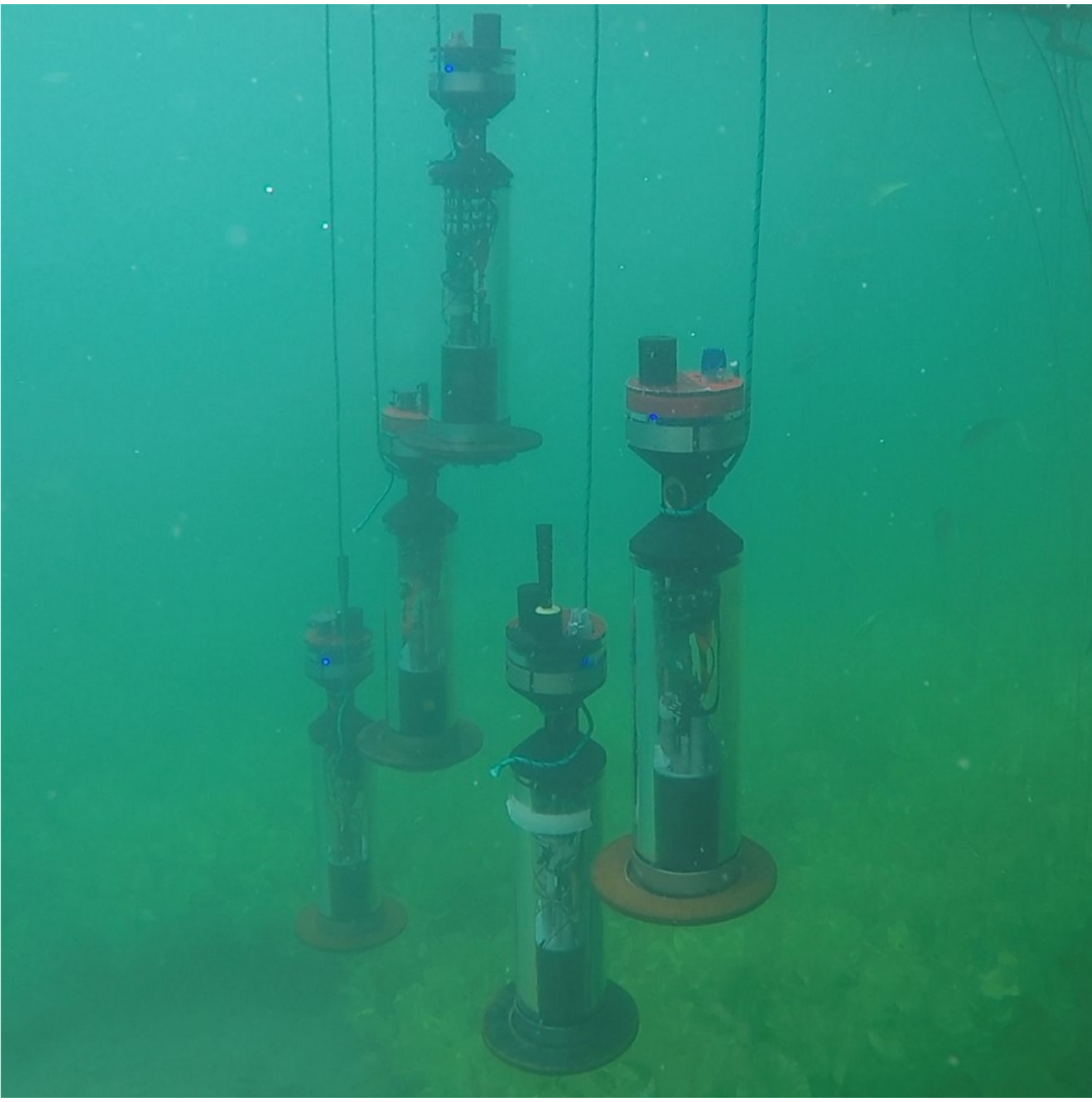}
	 \caption{Experiment setup using aMussel platforms}
                \label{fig:amussel}
\end{figure}
When on surface, aMussels can communicate either using WiFi, for data transfer, using Bluetooth, for diagnostics purposes, and using SMS, for sending short messages (like GPS coordinates) to greater distances. Each aMussel is also equipped with two underwater communication devices: green light, short-range communication device based on modulated light, and nanomodem, long-range acoustic communication device. In this work, we use only the later.

\subsection{Experiment Setup}
We have executed three experiments, one with 3 aMussels and two with 5 aMussels with the corresponding parameters in Table \ref{tab:parametri}. Since aMussels were close to each other, it is assumed that the graph is complete, corresponding to the connection matrix $\mathbf{A^1}$, but it is possible that some individual messages did not reach all the agents. 
In all experiments one communication step lasts for 5 seconds.

For the purpose of the experiment with 3 aMussels, measurements are kept constant until moment $k_{Step}$ when aMussel3 changes its measured value from 100 to 180, making a change in system median. This experiment was conducted with aMussels next to each other where for each aMussel, the measured value is a preprogrammed number that does not correspond to any physical value.

Experiments with 5 aMussels were conducted in Jarun lake in Zagreb. Each aMussel was tied to a rope with different length, so each of them was on different depth (Figure \ref{fig:amussel}). Since aMussels were placed close to each other on different depths, a significant loss of communication packages was expected. During the experiment, aMussels measured pressure in hPa, where one hPa corresponds approximately to one cm of depth. The surface pressure is around 1000 hPa. To enforce a change in measured values, the aMussel with the longest rope was moved towards the surface in the middle of the experiment, and thus became the aMussel with the smallest depth. The experiment goal was for aMussel to agree on the median of their depths (which corresponds to their measured pressure).

Additional experiment with 5 aMussels was conducted to show the influence of communication loss on the results. Same as in the previous experiments, aMussels were tied to a rope on different depths. After they reached consensus, one of them was removed from the water for 5 minutes and then returned to the same depth. After that, two aMussels were removed from the water and left to communicate between each other. After some time, they were returned to the water and placed on different depths. 

\subsection{Results}

Results of experiments with 3 aMussels (Figure \ref{fig:exp3}) show that values of individual aMussels converge towards the median values. After the measurements of the aMussel3 change, the median value of the group changes, which leads to convergence of individual aMussel values towards the new median value.  The results are slightly oscillatory compared to the simulation, because all values transferred over acoustics are quantizied. The experiment was run 4 times with similar results.

In the experiment with 5 aMussels, they first reach consensus which corresponds to the measurement of aMussel27 (Figure \ref{fig:exp5}). After the aMussel39 changes its depth, aMussels manage to reach the consensus around the measurement of aMussel18, which represents the new median value. The measured total loss of packages during the experiment was around 15\%, with aMussel39 having a loss of around 30\%. This experiment was run 2 times with similar results.

The results of communication loss experiment are shown in Figure \ref{fig:expcommloss}. After 600 steps aMussel27 was removed from the water for 5 minutes, and the results show that this short term communication loss did not influence the system. When aMussel31 and aMussel39 were removed from water and left to communicate only between each other, they started to agree on the consensus between only their measurements. After returning them into the water on different depths, all aMussels in the water started converging to the new median value. This experiment was run once.

\begin{figure}[h]
    \centering
    \includegraphics[width=0.7\linewidth,trim=100 300 110 300, clip]{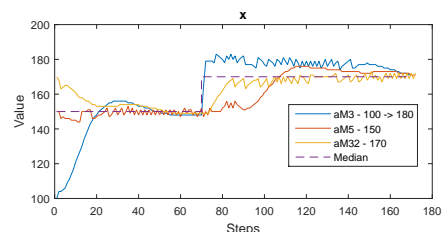}
    \caption{Experimental results for 3 aMussels}
        \label{fig:exp3}
\end{figure}

\begin{figure}[h]
    \centering
    \includegraphics[width=0.8\linewidth,trim=100 250 110 240, clip]{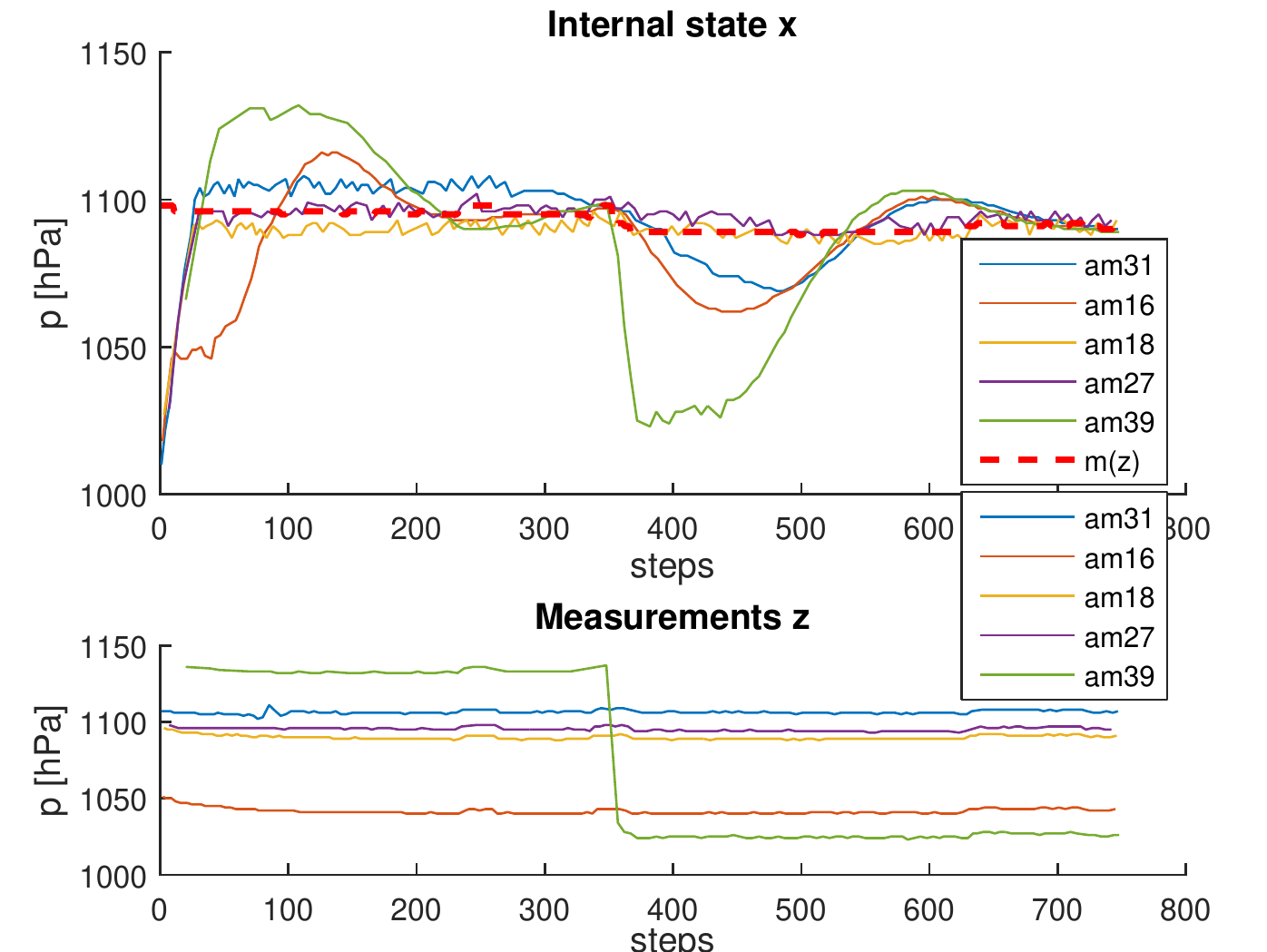}
    \caption{Experimental results for 5 aMussels}
        \label{fig:exp5}
\end{figure}

\begin{figure}[h]
    \centering
    \includegraphics[width=0.8\linewidth,trim=100 240 110 260, clip]{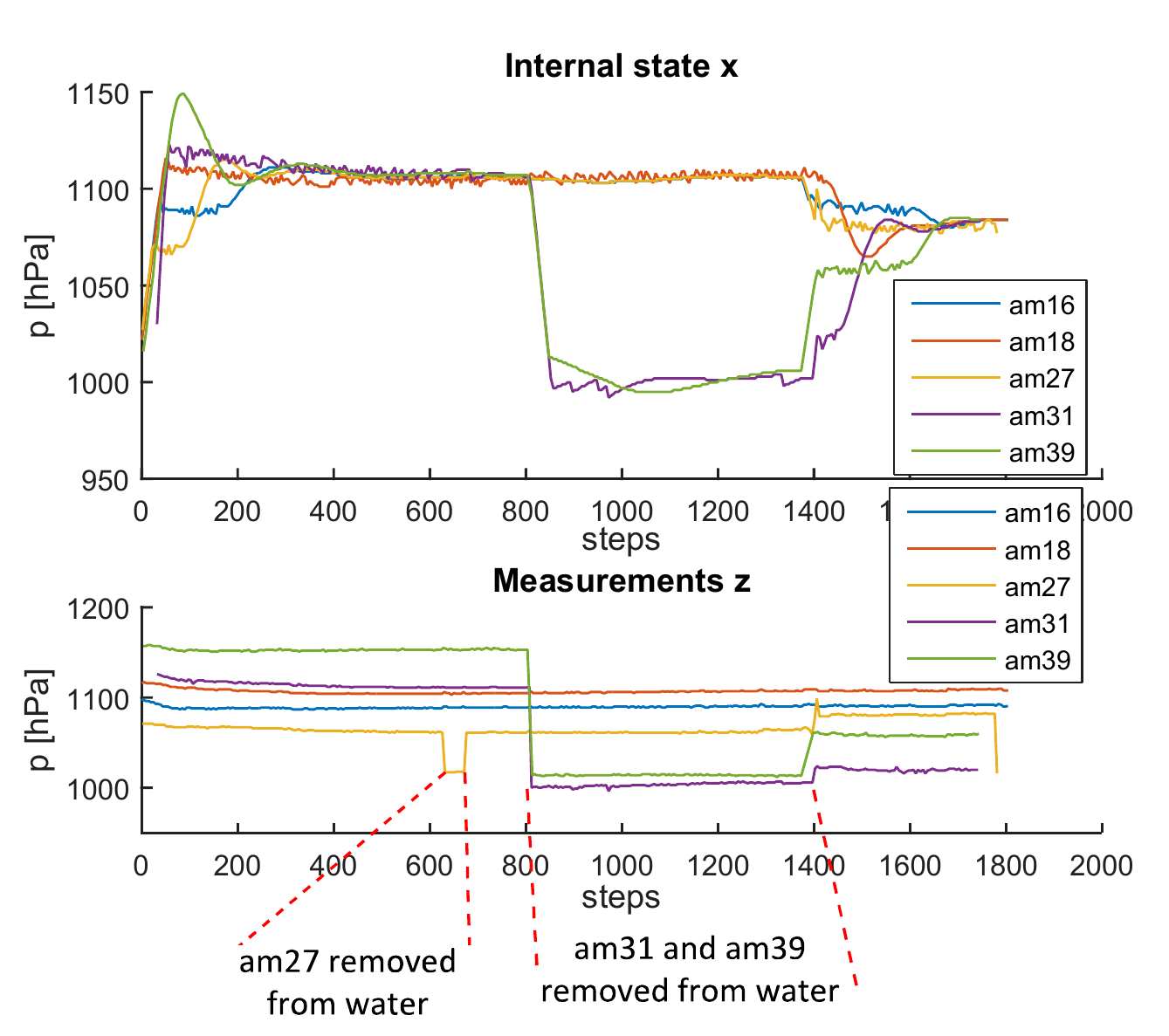}
    \caption{Experimental results for communication loss experiment}
        \label{fig:expcommloss}
\end{figure}

\section{Conclusion}
\label{sec:conclusion}

In this paper we have presented a method for determining the median value of measurements for a group of agents communicating using scheduled acoustic communication channel. The convergence of the protocol towards the median value is proven theoretically. In order to validate the presented protocol we tested it in a simulation setting, by creating a  model of the multi-agent system using scheduled communication. This model was used to gather simulation results of the presented dynamic median method for different number of agents, different connectivity matrices and tuning parameters.
We have also tested the presented method on the underwater robotic platforms aMussel, which are equipped with acoustic communication units. Simulation, as well as experimental results, have shown that the presented method converges towards the median of the measurements and that parameters of the consensus protocol can be tuned so that desired speed of convergence and accuracy of the system output obtain desired values. Results also show that the protocol work correctly even for a higher percentage of communication losses.

\addtolength{\textheight}{-12cm}   





\bibliographystyle{bibliography/IEEEtran}
\bibliography{IROS_RAL_2020_Median_Consensus.bib}

\end{document}